\documentclass[pdflatex,sn-aps]{sn-jnl}

\usepackage{graphicx}%
\usepackage{multirow}%
\usepackage{amsmath,amssymb,amsfonts}%
\usepackage{amsthm}%
\usepackage{mathrsfs}%
\usepackage[title]{appendix}%
\usepackage{xcolor}%
\usepackage{textcomp}%
\usepackage{manyfoot}%
\usepackage{booktabs}%
\usepackage{algorithm}%
\usepackage{algorithmicx}%
\usepackage{algpseudocode}%
\usepackage{listings}%
\usepackage{enumitem}
\usepackage{mathtools}

\theoremstyle{thmstyleone}%
\newtheorem{theorem}{Theorem}
\newtheorem{lemma}[theorem]{Lemma}%
\newtheorem{corollary}[theorem]{Corollary}%

\theoremstyle{thmstyletwo}%

\theoremstyle{thmstylethree}%
\newtheorem{definition}{Definition}%

\usepackage{enumitem}
\usepackage{amsmath,amssymb,amsfonts,amsthm}

\usepackage{booktabs}
\usepackage{enumitem}

\usepackage{multicol,multirow}
\usepackage{url}
\usepackage{xspace}
\usepackage{xcolor}
\usepackage{xfrac}
\usepackage[table]{xcolor}

\usepackage{algorithm}

\usepackage{tikz}
\usetikzlibrary{arrows}
\usetikzlibrary{shapes}
\usepackage{subcaption}

\usepackage{forest}
\definecolor{folderbg}{RGB}{124,166,198}
\definecolor{folderborder}{RGB}{110,144,169}

\def\Size{4pt}
\tikzset{
  folder/.pic={
    \filldraw[draw=folderborder,top color=folderbg!50,bottom color=folderbg]
      (-1.05*\Size,0.2\Size+5pt) rectangle ++(.75*\Size,-0.2\Size-5pt);  
    \filldraw[draw=folderborder,top color=folderbg!50,bottom color=folderbg]
      (-1.15*\Size,-\Size) rectangle (1.15*\Size,\Size);
  }
}

\usepackage{listings}
\newcommand{\pmv}[2]{#1\,\scriptsize{$\pm$}\,#2}

\newcommand{\method}{\text{XOR-SMOO}\xspace}

\begin{document}

\title[Article Title]{Approximating Pareto Frontiers in Stochastic Multi-Objective Optimization via Hashing and Randomization}


\author[1]{\fnm{Jinzhao} \sur{Li}}\email{li4255@purdue.edu}

\author[2]{\fnm{Nan} \sur{Jiang}}\email{njiang@utep.edu}

\author[1]{\fnm{Yexiang} \sur{Xue}}\email{yexiang@purdue.edu}

\affil[1]{\orgdiv{Department of Computer Science}, \orgname{Purdue University}, \orgaddress{\street{610 Purdue Mall}, \city{West Lafayette}, \postcode{47907}, \state{IN}, \country{United States}}}

\affil[2]{\orgdiv{Department of Computer Science}, \orgname{University of Texas at El Paso}, \orgaddress{\street{500 W University Ave}, \city{El Paso}, \postcode{79968}, \state{TX}, \country{United States}}}


\abstract{

Stochastic Multi-Objective Optimization (SMOO) is critical for decision-making trading off multiple potentially conflicting objectives in uncertain environments. 
SMOO aims at identifying the Pareto frontier, which contains all mutually non-dominating decisions. The problem is highly intractable due to the embedded probabilistic inference, such as computing the marginal, posterior probabilities, or expectations. 
Existing methods, such as scalarization, sample average approximation, and evolutionary algorithms, either offer arbitrarily loose approximations or may incur prohibitive computational costs.
We propose \method, a novel algorithm that with probability $1-\delta$, 
obtains $\gamma$-approximate Pareto frontiers ($\gamma>1$) for SMOO by querying an SAT oracle poly-log times in $\gamma$ and $\delta$.
A $\gamma$-approximate Pareto frontier is only below the true frontier by a fixed, multiplicative factor $\gamma$. 
Thus, \method solves highly intractable SMOO problems (\#P-hard) with only queries to SAT oracles while obtaining tight, constant factor approximation guarantees. 
%
Experiments on real-world road network strengthening and supply chain design problems demonstrate that \method outperforms several baselines in identifying Pareto frontiers that have higher objective values, better coverage of the optimal solutions, and the solutions found are more evenly distributed.
Overall, \method significantly enhanced the practicality and reliability of SMOO solvers.

}




\maketitle

\section{Introduction}









Trading off multiple, often conflicting objectives is a central problem in economics, operational research, and AI. For example, in many real-world domains, such as supply chain planning~\cite{AltiparmakGLP06,YuG14,pishvaee2012environmental}, network design~\cite{DBLP:journals/eswa/OwaisO18,miandoabchi2013multi}, energy deployment~\cite{clarke2015multi,bi2018wireless}, and path planning~\cite{nazarahari2019multi}, decision makers must simultaneously optimize several criteria (e.g., cost, reliability, efficiency), rather than focusing on a single objective. 
In such scenarios, it is rare that one solution dominates all objectives.

The goal of \emph{multi-objective optimization} is to characterize the set of solutions that trade off among objectives.
A standard concept is \emph{Pareto optimality}. A solution is Pareto optimal if no other feasible solution improves one objective without worsening another. The collection of all such solutions forms the \emph{Pareto frontier}, which provides a compact representation of the best achievable trade-offs.
See Figure \ref{fig:intro} (Left) for a visual example involving two objectives. No one point in the Pareto frontier, plotted as the orange line, dominates others in both objectives. 
Because exact Pareto frontiers are often expensive to compute, a rich line of work studies \emph{$\gamma$-approximate Pareto frontiers}, where every Pareto-optimal solution is approximated by a solution in the approximate frontier, and every objective value of these two solutions is within a constant, multiplicative factor $\gamma$ ($\gamma > 1$). In Figure \ref{fig:intro} (Left), the upper hull of all green points makes up a $\gamma$-approximate frontier. This is because every point in the true Pareto frontier is within a multiplicative distance $\gamma$ of at least one green point.

A classical result by Papadimitriou and Yannakakis~\cite{papadimitriou2000approximability} shows that 
a $\gamma$-approximate Pareto frontier can be constructed by querying a SATisfiability (SAT) oracle $O(1/(\log \gamma)^k)$ times, where $k$ is the number of objectives.
However, in many practical applications, objectives are inherently stochastic, leading to Stochastic Multi-Objective Optimization (SMOO). 
%
Objective functions in SMOO are expectations or posterior or marginal probabilities over random variables. 
Evaluating a single objective of this type requires probabilistic inference over exponentially many probabilistic scenarios. Theoretically, such problems are \#P-complete~\cite{DBLP:journals/ai/Roth96,DBLP:journals/ai/ChaviraD08,DBLP:journals/tcs/Valiant79}. This renders SMOO problems highly intractable.


Existing approaches to SMOO typically apply scalarization techniques that reduce multiple objectives to a single one~\cite{afrouzy2016fuzzy,fliege2011stochastic}. These approaches fail to capture the full trade-offs among objectives. Others rely on sample average approximation (SAA)~\cite{bonnel2014stochastic,karimi2021biobjective,mercier2019non} to estimate expectations. These methods lack performance guarantees as it may take exponentially many steps for common samplers, such as Markov Chain Monte-Carlo (MCMC), to mix. Several additional methods are tailored to specific formulations, for example, assuming convex surrogates~\cite{liu2019stochastic}, or relying on mixed-integer or dynamic programming~\cite{sheidaei2021stochastic,wu2018efficiently}, limiting their general applicability.

This paper proposes a new algorithm, \method, that obtains \textbf{\emph{$\gamma$-approximate Pareto frontiers} for Stochastic Multi-Objective Optimization (SMOO)} problems. With probability $1-\delta$, \method will obtain such an approximate Pareto frontier by querying an SAT oracle 
$O\bigl(((|Y|+\log U+\log\frac{1}{\gamma-1})/(\gamma-1))^k\bigr)$ times. 
%
Here, $|Y|$ is the maximum number of binary random variables to evaluate a stochastic objective. $U$ is the range of the stochastic objectives. 
To our knowledge, this is the \textbf{first algorithm} that  approximates the Pareto frontiers of highly intractable SMOO problems up to any constant $\gamma>1$ (\#P-hard) utilizing accesses to NP oracles. 
$\gamma$ can be made to be arbitrarily close to 1 (and $\delta$ close to 0) with the availability of computing resources.

\begin{figure}[t]
    \centering
    \includegraphics[width=0.85\linewidth]{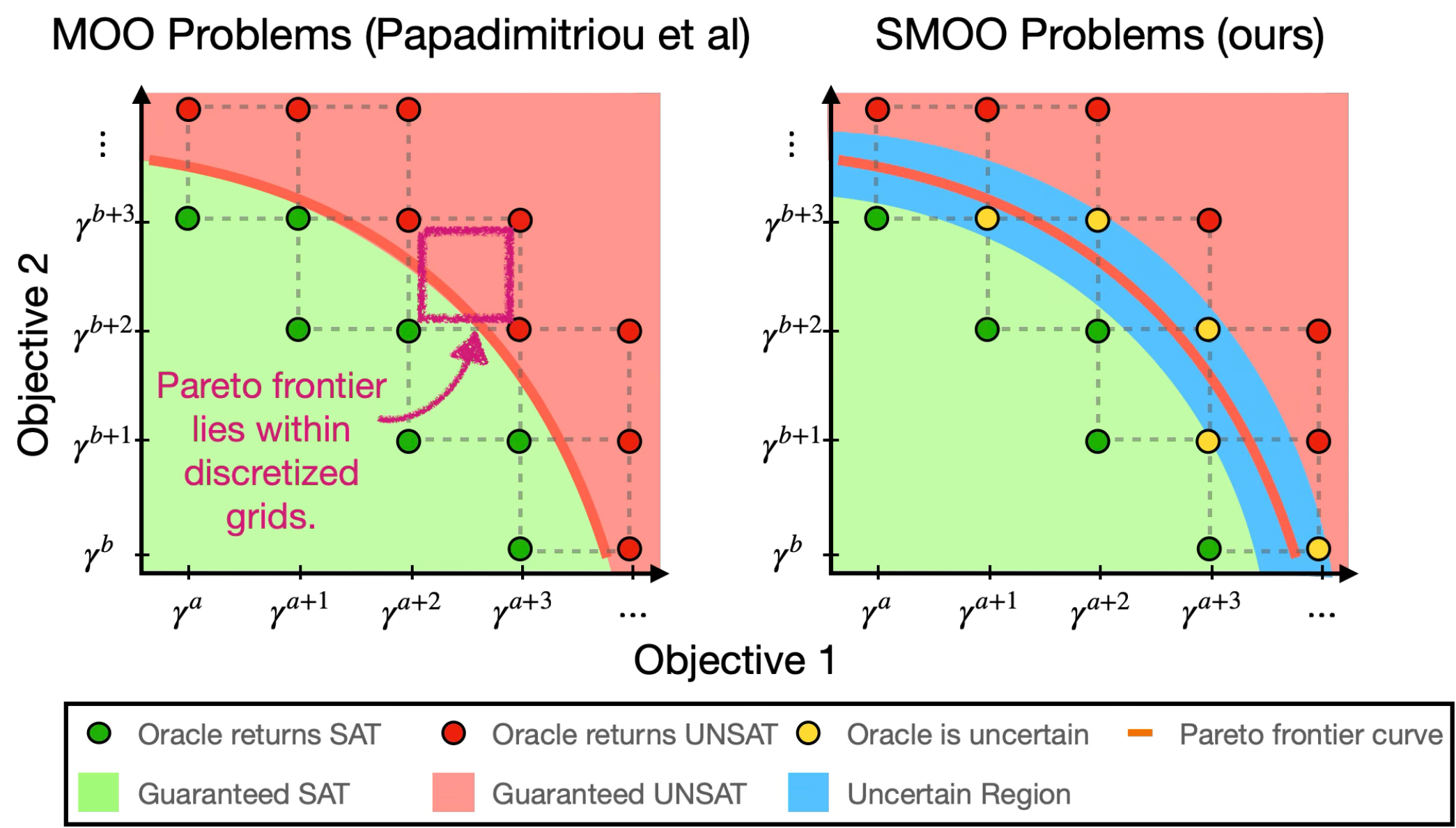}
    \caption{ 
    Solving SMOO problems via querying satisfiability oracles.
    (\textbf{Left}) In multi-objective optimization (MOO), Papadimitriou et al.~\cite{papadimitriou2000approximability} lays down a multiplicative grid, where adjacent grid points are separated by $\gamma$. 
    Then, for every grid point, a SAT oracle is queried 
to determine if a solution exists such that each of its objective value exceeds
the grid point’s value at the corresponding dimension. 
    SAT oracles' responses split the entire region into a (\textcolor{green}{SAT}) and a (\textcolor{red}{UNSAT}) region. 
    The top of all green points form a $\gamma$-approximate Pareto frontier. 
    (\textbf{Right}) In SMOO, 
    with high probability, our developed  probabilistic oracle makes the correct
SAT/UNSAT decisions only when the objective values exceed (or lack behind) the
queried threshold by a fixed multiplicative constant. This brings in a third, intermediate uncertain region (shown in \textcolor{blue}{blue}). However, because its width can be controlled, the top of all green points still form a $\gamma$-approximate Pareto frontier. 
}
    \label{fig:intro}
\end{figure}

The proposed \method requires two ingredients. The first is provable probabilistic inference (model counting) via \emph{hashing and randomization}. 
We need this technology to bound each stochastic objective tightly and with confidence. 
The unweighted version of probabilistic inference is to count the number of models (i.e., solutions) to a SAT formula, often known as model counting~\cite{DBLP:journals/tcs/Valiant79,DBLP:journals/siamcomp/Valiant79}. 
Counting via hashing and randomization originates from Valiant's seminal work on unique SAT~\cite{Valiant1986uniqueSAT,jerrum1986random} and later developed into hashing-based model counting~\cite{Gomes2006Sampling,Gomes06XORCounting,Ermon13Wish,ermon2013embed,kuck2019adaptive,Chakraborty2013scalable,Chakraborty2014DistributionAwareSA,boreale2019approximate, tan2024formally}.
With the advent of efficient SAT solvers~\cite{Braunstein2005SurveyPA}, this approach has become both theoretically sound and practically scalable.
To decide if a SAT formula $f(x)$ has more than $2^l$ solutions, we consider if $f(x) \wedge \mathtt{XOR}_1(x) \wedge \cdots \wedge \mathtt{XOR}_l(x)$ is satisfiable.
Each constraint $\mathtt{XOR}_i(x)$ is the logical XOR of a randomly sampled subset of variables from $x$. It can be interpreted as the parity of sampled variables. 
Intuitively, each sampled $\mathtt{XOR}$ constraint rules out half of the solutions of $f(x)$. Hence, if $f(x)$ has more than $2^l$ solutions, after adding $l$ $\mathtt{XOR}$ constraints, the formula should still be satisfiable, and vice versa. 
At a high level, this approach gives us a probabilistic SAT oracle that can probe model counts with constant-factor precision (up to factors of 2). 
Counting via hashing and randomization will need to be adapted for \method. In particular, we will need to control the failure probabilities across different thresholds for multiple objectives, and to devise a discretization scheme that yields an arbitrarily close-to-1 approximation to weighted problems. 

The second ingredient behind \method is to modify the discretization schema of Papadimitriou and Yannakakis~\cite{papadimitriou2000approximability} to fit in probabilistic SAT oracles.
To sketch a $\gamma$-approximate Pareto frontier, Papadimitriou and Yannakakis~\cite{papadimitriou2000approximability}'s idea is to lay a $k$-dimensional grid, where two adjacent grid points along one dimension are separated by the fixed multiplicative factor, $\gamma$.
Then, for every grid point, the SAT oracle is to determine if there exists a solution such that each of its objective value exceeds the grid point's value at the corresponding dimension.
The queried results will split all grid points into two regions: the SAT region, where such solutions exist, and the UNSAT region. See Figure \ref{fig:intro} (Left) for an example. The true Pareto frontier must be sandwiched between the red points in the UNSAT region and the green points in the SAT region.
Because every adjacent pair of (red, green) points is only $\gamma$ distance apart, the boundary that splits the SAT and UNSAT regions forms a {$\gamma$-approximate Pareto frontier}. 

Probabilistic oracles used in \method complicates our analysis by bringing a third, intermediate \emph{uncertain} region between the SAT and UNSAT regions (Figure \ref{fig:intro} right).
This is because with high probability, the probabilistic oracle makes the correct SAT/UNSAT decisions only when the objective values exceed (or lack behind) the queried threshold by a fixed multiplicative constant. 
Inside the uncertain region, the SAT/UNSAT decisions are not informative. 
However, the width of the uncertain region is limited. This allows us to prove that the lower boundary of the uncertain region still serves as a good approximate Pareto frontier. 


%

%
%

We first devise \method on unweighted multi-objective problems as a ladder towards weighted problems. 
%
Each unweighted objective is a model counting problem. 
%
%
In this case, with probability $1-\delta$, \method can produce an approximate Pareto frontier represented in objective values, such that the true frontier is at most $2^{\epsilon}$ multiplicative factor away.
Or, \method produces a set of solutions that forms a $2^{2\epsilon-1}$-approximate Pareto frontier with probability $1-\delta$. Here $\epsilon \geq 3$, but the failure probability $\delta$ can be close to 0. \footnote{
Using the techniques presented for the weighted objectives can further tighten these bounds. However, for readability, we do not introduce those techniques until the weighted objective section. }
\method reduces the SMOO problem to a set of SAT queries, where the number of queries scales with the product of the number of latent variables (i.e., those being summed over) for each objective.
Assuming each unweighted objective is represented in an SAT encoding, \method needs to solve SAT instances whose size is $O(n+\log \frac{1}{\delta} +k \log|Y|)$ times the size of that encoding, where $n$ is the number of decision variables, $\delta\in (0,1)$ is the error probability bound, $|Y|$ is the maximum number of random variables used to evaluate a stochastic objective, 
and $k$ is the number of objectives. 


For weighted objectives, we extend \method to w-\method, which finds an arbitrarily small $\gamma$-approximate Pareto frontier, for any $\gamma > 0$. 
The key idea is to construct a pseudo-unweighted SMOO problem that mirrors the original weighted problem. Then the approximation guarantee obtained by the unweighted algorithm can be carried over to the weighted problem. 
Our first approximation is to round each weighted objective from below to its nearest $2^{b_0}$ power. 
Then we use $b$ ($b\geq b_0$) binary variables $z_0, \ldots, z_{b-1}$, in which $z_0, \ldots, z_{b_0-1}$ can take both values 0 and 1, but $z_{b_0}, \ldots, z_{b-1}$ are restricted to take value 0. This ensures that the number of different configurations of these binary variables is $2^{b_0}$, within a constant factor of the original weighted objective. 
The second step is to tighten the approximation bound. Suppose the weighted objective is $\sum_y f(x,y)$,  obtaining $2^{2\epsilon}$-approximate Pareto frontier for $(\sum_y f(x,y))^T$ gives us a $2^{2\epsilon/T}$-approximate frontier for the original weighted objective. 
Overall, to obtain $\gamma$-approximate Pareto frontiers for weighted SMOO problems with probability $1-\delta$, 
w-\method queries an SAT oracle 
$O\bigl(((|Y|+\log U+\log\frac{1}{\gamma-1})/(\gamma-1))^k\bigr)$ times, and each SAT instance's size is 
$O(n+\log\tfrac{1}{\delta}+k\log(|Y|+\log U))$
times the size of the SAT encoding of each objective. Here $U$ is the maximal range of stochastic objectives.



We compare \method with state-of-the-art multi-objective solvers on two applications: \textit{Road Network Strengthening to Mitigate Seasonal Disruptions}, and \textit{Flexible Supply Chain Network Design}. 
Both applications are grounded in real-world or standard benchmark data sources. The first is constructed from OpenStreetMap road networks and geographically grounded weather records from the Meteostat library. The second is derived from widely used TSPLIB benchmark instances.
Experimental results show that our \method consistently finds \textbf{\textit{the best Pareto solutions}}, meaning that our method finds solutions that have the best objective values among those found by all solvers. 
\method also achieves \textbf{\textit{the best coverage}}: for every Pareto optimal solution, our method is more likely to find one closely approximating it. 
Finally, our \method finds \textbf{\textit{the most evenly distributed solutions}}: the solutions found by
our method spread out evenly in the entire domain, hence capturing the widest portion of the Pareto frontier. 
Moreover, the performance gap becomes more pronounced as the counting objectives become more difficult, highlighting the advantage of our proposed XOR-SMOO solver. 

%


\section{Preliminaries}
\subsection{Multi-Objective Optimization}
A multi-objective optimization (MOO) problem is defined as
\begin{align*}
    \max_{x \in \mathcal{X}}~~  (f_1(x), \dots, f_k(x)),
\end{align*}
where $\mathcal{X}$ denotes the set of feasible solutions, also called the \emph{solution space} or \emph{decision space}, and each $f_i : \mathcal{X} \rightarrow \mathbb{R}$ is an objective function, for $i = 1,\dots,k$.

We use the notation $\max (f_1, \dots, f_k)$ to represent the maximization of multiple functions. In practice, there may not exist one $x^*$ which attains the maximal value of all $k$ functions.
Hence, trading off the values of one function with others is necessary. This leads to reasoning about the Pareto frontier, which will be discussed momentarily.  
%
Another commonly used approach to reduce multiple objectives into a single one is \emph{scalarization}, which concatenates functions with an affine function. 
The formulation used in this paper is based on \emph{Pareto optimality}, which characterizes the trade-offs among multiple objectives with a set of mutually non-dominating solutions. 

\begin{definition}[Pareto Frontier]
Consider a multi-objective optimization problem with $k$ objectives $\{f_i\}_{i=1}^k$ to be maximized.  
For two solutions $x_1, x_2 \in \mathcal{X}$, we say that $x_1$ \emph{dominates} $x_2$, written as $x_1 \succ x_2$, if
\begin{equation}
    f_i(x_1) \geq f_i(x_2), \quad \text{for all } i = 1, \dots, k,
\end{equation}
and strict inequality holds for at least one index $i$.  

A solution $x^* \in \mathcal{X}$ is \emph{Pareto optimal} (or \emph{non-dominated}) if there exists no $x \in \mathcal{X}$ such that $x \succ x^*$.  
The set of Pareto optimal solutions forms the \emph{Pareto frontier}.
\end{definition}

A Pareto optimal solution implies that no actions or allocations are possible to improve one objective without affecting other ones. 
The Pareto frontier exactly characterizes all locally optimal trade-offs. 
In practice, computing the exact frontier is often intractable due to the exponential size of the search space.  
A common relaxation is to compute an \emph{approximate Pareto frontier}, which allows for small multiplicative deviations from the exact frontier. 

\begin{definition}[$\gamma$-Approximate Pareto Frontier~\cite{papadimitriou2000approximability}]
\label{def:approx_pareto}
Consider a multi-objective optimization problem with $k$ objectives $\{f_i\}_{i=1}^k$ to be maximized.  
For $\gamma > 1$ and two solutions $x_1, x_2 \in \mathcal{X}$, we say that $x_1$ \emph{$\gamma$-dominates} $x_2$ if
\[
    \gamma f_i(x_1) \geq f_i(x_2), \quad \text{for all } i = 1, \dots, k.
\]
A set $\widehat{\mathcal{F}} \subseteq \mathcal{X}$ is called an \emph{$\gamma$-approximate Pareto frontier} if for every Pareto optimal solution $x \in \mathcal{F}$,  
there exists some $x' \in \widehat{\mathcal{F}}$ such that $x'$ $\gamma$-dominates $x$.  
\end{definition}

In other words, an $\gamma$-approximate Pareto frontier guarantees that every true Pareto optimal solution has an approximate representative in $\widehat{\mathcal{F}}$. Each objective of this true optimal solution is within a multiplicative factor $\gamma$ of the corresponding one of the approximate representative. This relaxation allows for efficient computation while preserving the trade-offs among solutions in the Pareto frontier.

\subsection{Stochastic Multi-Objective Optimization}
A stochastic multi-objective optimization (SMOO) problem~\cite{mahdavi2013stochastic} arises when stochastic events affect multiple objective values, and decisions must be made prior to observing these random events~\cite{gutjahr2016stochastic}.  
For example, in an asset allocation problem in a stochastic trading market, one must simultaneously maximize the expected returns and minimize the asset volatility, while accounting for random price fluctuation, etc.

Formally, let $x$ denote decision variables (the amount of asset in the portfolio) and $y$ denote random variables drawn from the domain $\mathcal{D}$ (the asset prices). 
$f$ is the target objective (in our example, the total profit of the asset portfolio).
Because of the randomness represented in variables $y$, the optimization typically involves maximizing the expected value of the target objective:
\begin{align*}
    \max_{x\in \mathcal{X}}~~\mathbb{E}_{y \sim \mathcal{D}}[f(x,y)].
\end{align*}
Extending this to the multi-objective setting with $k$ objectives (for example, $f_1$ is the asset profit, $f_2$ is the volatility), a general SMOO problem can be written as
\begin{align*}
    \max_{x\in \mathcal{X}} ~~\left( \mathbb{E}_{y_1 \sim \mathcal{D}_1}[f_1(x,y_1)], \dots, \mathbb{E}_{y_k \sim \mathcal{D}_k}[f_k(x,y_k)] \right). 
\end{align*}
The Pareto frontier in this context consists of all non-dominated solutions in expected objective values. 
Note that stochasticity may affect the shape of the constrained region as well (e.g., via randomized constraints). 
In this work, however, we restrict our attention to maximizing the expected values. 
Randomized constraints can be encoded into the objective function using, for example, the $\lambda$-multipliers. 

\subsection{Probabilistic Inference and Model Counting}

Probabilistic inference, for example, the computation of expectations, marginal probabilities, posterior probabilities, 
can be encoded as weighted model counting~\cite{chavira2008probabilistic,Xue2016MarginalMAP}. 
Let us start our discussion on the unweighted case. 
Unweighted model counting computes the number of satisfying solutions to a Boolean formula. 
Formally, let $f({x})$ be a Boolean function over ${x} \in \{0,1\}^n$, where $f({x}) = 1$ denotes that ${x}$ satisfies the formula.  
The unweighted model counting problem computes $\sum_{{x \in \{0,1\}^n}} f({x})$.

The weighted model counting problem computes 
the sum of an arbitrary weighted function. For example, let $f$ be a function that maps $\{0, 1\}^n$ to $\mathbb{R}^+$. The weighted version computes $\sum_{{x} \in \{0, 1\}^n} f({x})$. 
Various probabilistic inference can be reduced to this summation. For example, computing the expectation, 
\[
\mathbb{E}_{y \sim Pr(y|x)}[f(x,y)]  =  \sum_{y} \Pr(y|x) f(x,y),
\]
is a weighted model counting problem. 




To improve readability, we will first detail our approximate SMOO solver and the approximation guarantee assuming unweighted model counting objectives (Section \ref{sec:method}), then progress to weighted problems (Section \ref{sec:extension}). 

\subsection{Solving Model Counting using Hashing and Randomization}
\label{sec:prelim-xor}

It is highly intractable to solve model counting. 
Unlike satisfiability, which decides the existence of one satisfying assignment, model counting requires estimating the total number of satisfying assignments, and is $\#\mathrm{P}$-complete. 
The complexity class $\#\mathrm{P}$ is believed to be beyond $\mathrm{NP}$, because it subsumes the entire polynomial hierarchy. 
Various model counting approaches have been proposed in the past. 
\emph{Exact} approaches include DPLL-style solvers \cite{sharma2019ganak,thurley2006sharpsat,sang2004combining,oztok2018exhaustive} and knowledge-compilation methods that transform a formula into tractable representations \cite{Muise2012dsharp,lagniez2017improved,darwiche2011sdd}.
\emph{Approximate} counters include variational approaches based on mean-field or belief-propagation relaxations \cite{kroc2008leveraging,KerstingAN09}, as well as sampling-based methods such as importance sampling \cite{gogate2007approximate} and MCMC-based techniques \cite{wei2005new}, which estimate the model count from sampled satisfying assignments.  
While these methods often scale well in practice, they give no guarantee or one-sided guarantees on the model counts that can be arbitrarily loose. 

A line of recent approaches approximate model counts via \emph{hashing and randomization}. 
These methods reduce model counting to SAT problems using randomized XOR constraints.
This idea originates from Valiant's seminal work on unique SAT~\cite{Valiant1986uniqueSAT,jerrum1986random} and later developed into hashing-based model counting~\cite{Gomes2006Sampling,Gomes06XORCounting,Ermon13Wish,ermon2013embed,kuck2019adaptive,Chakraborty2013scalable,Chakraborty2014DistributionAwareSA,boreale2019approximate, tan2024formally}.
With the advent of efficient SAT solvers~\cite{Braunstein2005SurveyPA}, this approach has become both theoretically sound and practically scalable.
%
The high-level idea is as follows. 
For example, suppose $x$ is fixed at $x_0$, and we would like to determine whether 
\begin{equation}
\sum_{y \in \{0,1\}^{|y|}} f(x_0, y) \lessgtr 2^l.
\label{eq:count}
\end{equation}
Consider the SAT formula
\begin{align}
    f(x_0, y) \wedge \mathtt{XOR}_1(y) \wedge \cdots \wedge \mathtt{XOR}_l(y),
    \label{eq:xor_q}
\end{align}
where $\mathtt{XOR}_1,\ldots,\mathtt{XOR}_l$ are randomly sampled XOR constraints.  
Each constraint $\mathtt{XOR}_i(y)$ is the logical XOR of a randomly sampled subset of variables from $y$. It can be interpreted as the parity of sampled variables. In other words, $\mathtt{XOR}_i(y)$ is true if and only if an odd number of these randomly sampled variables in the subset are true.

We can show that Formula~\eqref{eq:xor_q} is \emph{likely satisfiable} when the model count
$\sum_{y} f(x_0,y)$ exceeds $2^{l+l^*}$, and \emph{likely unsatisfiable} when it is below $2^{l-l^*}$, where $l^*$ is an integer at least 2. The intuition is as follows, random XOR constraints serve as universal hash functions: each constraint retains roughly half of the assignments $y$ for which $f(x_0,y)=1$, so $l$ independent constraints
partition the space of $y$ into $2^{l}$ nearly equal buckets. Checking the satisfiability
of \eqref{eq:xor_q} is therefore equivalent to asking whether the bucket determined by the sampled XORs
contains a satisfying assignment of $f(x_0,y)$. If $\sum_y f(x_0,y)\ge 2^{l+l^*}$, in other words, the assignments outnumber the buckets, with high probability the chosen bucket
contains at least one solution. On the other hand, if $\sum_y f(x_0,y)<2^{l-l^*}$, in other words, the buckets are more than the assignments, it is likely that a randomly picked bucket is empty. The next
lemma formalizes this approximation guarantee.



\begin{algorithm}[!h]
\caption{\texttt{XOR-Counting($f$, $l$, $x_0$)}}
\label{alg:xor-counting}
\begin{algorithmic}[1]
\Require{Boolean formula $f$; Number of XOR constraints $l$; fixed $x_0$.}
\State Randomly sample $\mathtt{XOR}_1(y), \ldots, \mathtt{XOR}_l(y)$
\State $\psi(x_0,y) \gets f(x_0, y) \wedge \mathtt{XOR}_1(y) \wedge \cdots \wedge \mathtt{XOR}_l(y)$
\If{$\psi(x_0,y)$ is satisfiable}
   \State \Return \texttt{True}.
\Else
   \State \Return \texttt{False}.
\EndIf
\end{algorithmic}
\end{algorithm}

\begin{lemma}[XOR Counting~\cite{jerrum1986random,Gomes06XORCounting,Ermon13Wish}]
\label{lem:xor-counting}
Given a Boolean function $f(x, y)$ and $l \in \mathbb{Z}_{\geq 0}$. 
Fix an assignment $x$ at $x_0 \in \{0,1\}^n$ and let $l^* \ge 2$. Then:
\begin{itemize}[align=left, leftmargin=0pt, labelwidth=0pt, itemindent=!]
    \item If $\sum_{y} f(x_0, y) \ge 2^{l + l^*}$, then with probability at least $1 - \tfrac{2^{l^*}}{(2^{l^*}-1)^2}$, $\texttt{XOR-Counting}(f, l, x_0)$ returns \texttt{True}.
    \item If $\sum_{y} f(x_0, y) \le 2^{l - l^*}$, then with probability at least $1 - \tfrac{2^{l^*}}{(2^{l^*}-1)^2}$, $\texttt{XOR-Counting}(f, l, x_0)$ returns \texttt{False}.
\end{itemize}
\end{lemma}

\subsection{$\gamma$-Approximate Pareto Frontier via Discretization and Satisfiability Solving}

A common technique to connect optimization with satisfiability is to reformulate maximization as a sequence of threshold queries.  
Instead of directly maximizing an objective function $f(x)$, one repeatedly checks whether there exists a feasible solution $x \in \mathcal{X}$ such that $f(x) \geq Q$ for a threshold $Q$.  
By gradually increasing $Q$ until the query becomes infeasible, the maximum achievable value of $f(x)$ can be identified. In practice, this increase is performed in discrete steps rather than continuously, some precision may be lost, and the method in general yields an approximate rather than exact solution.

In the multi-objective setting, this idea extends to vector thresholds $(Q_1, \dots, Q_k)$, where the task is to decide whether all objective functions simultaneously achieve their respective thresholds.  
This reformulation transforms a multi-objective optimization problem into a family of decision problems.  
The following classical result formalizes how discretized threshold queries can approximate the Pareto frontier.

\begin{theorem}[Papadimitriou and Yannakakis~\cite{papadimitriou2000approximability}]\label{th:papadim}
Let $\gamma > 1$ be a constant, and consider a $k$-objective maximization problem:
\begin{align*}
    \max_{x\in\mathcal{X}} \big( f_1(x), f_2(x), \dots, f_k(x) \big).
\end{align*}
Suppose we search for integer tuples $(q_1, \dots, q_k) \in \mathbb{N}^k$ such that:
\begin{itemize}
    \item There exists a feasible solution $x^*$ with $f_i(x^*) \ge \gamma^{q_i}$ for all $i \in \{1, \dots, k\}$.
    \item For every $(q_1', \dots, q_k') \in \mathbb{N}^k$ with $q_i' \ge q_i$ for all $i$ and $q_j' > q_j$ for at least one index $j$, no feasible solution $x$ satisfies $f_i(x) \ge \gamma^{q_i'}$ for all $i$.
\end{itemize}
Then the set of such solutions $x^*$ constitutes the $\gamma$-approximate Pareto frontier.
\end{theorem}


The key idea driving the work of Papadimitriou and Yannakakis~\cite{papadimitriou2000approximability} is to impose a $\gamma$-multiplicative grid discretization over the $k$-dimensional objective space.  
For any grid point $(\gamma^{q_1},\dots,\gamma^{q_k})$, a SAT query checks whether there exists a solution whose objective functions meet all these thresholds.  

Every Pareto-optimal solution must lie inside some grid cell.  
If we round each objective value $f_i(x_{\mathrm{opt}})$ of a Pareto-optimal point $x_{\mathrm{opt}}$ down to the nearest grid level $\gamma^{q_i}$, the resulting threshold vector corresponds to a satisfiable query. 
Because adjacent grid levels differ by exactly a factor of $\gamma$, any solution ($x^*$ in Theorem~\ref{th:papadim}) satisfying the rounded-down threshold achieves each objective within a multiplicative factor of $\gamma$ of the true Pareto-optimal values.  
Thus, the solutions associated with all maximal satisfiable grid points collectively form the $\gamma$-approximate Pareto frontier.


\section{SMOO Problem Formulation}
In this paper, we study SMOO problems involving $k$ objectives. Formally, the problem can be expressed as
\begin{align}
    &\max_{x \in \{0,1\}^n} 
    \left( \sum_{y_1} f_1(x,y_1), \ldots , \sum_{y_k} f_k(x,y_k) \right)
    \label{eq:smoop}
\end{align}
where:
\begin{itemize}
    \item $x \in \{0,1\}^n$ is a vector of binary decision variables.
    \item $y_i \in \{0,1\}^{|y_i|}$ are latent binary variables that capture stochasticity through model counting.
    \item $f_i : \{0,1\}^{n + |y_i|} \rightarrow \mathbb{R}_{\ge 0}$ are functions defined over both decision and latent variables. 
\end{itemize}
Each term $\sum_{y_i} f_i(x, y_i)$ represents a model-counting–based objective, where the summation over the latent variables $y_i$ captures the underlying stochasticity. 
Depending on the application, $f_i$ can take two forms:
\begin{itemize}
    \item \emph{Unweighted functions}: $f_i(x, y_i) \in \{0,1\}$, where the summation counts the number of satisfying configurations of $y_i$ given $x$. 
    \item \emph{Weighted functions}: $f_i(x, y_i) \in \mathbb{R}_{\ge 0}$, where each configuration of $(x, y_i)$ contributes a weight, corresponding to probabilistic or expectation-based objectives.
\end{itemize}
The unweighted case serves as the foundation of our approach and is discussed first, while the weighted extension is introduced in a later section. 
We assume that all model-counting objectives in Equation~\eqref{eq:smoop} are computationally intractable to evaluate exactly, which necessitates the development of approximate methods with theoretical performance guarantees.
\section{Solving Unweighted SMOO Problems}
\label{sec:method}

\begin{figure}[p]
    \vspace{-2em}
    \centering
    \includegraphics[width=0.95\linewidth]{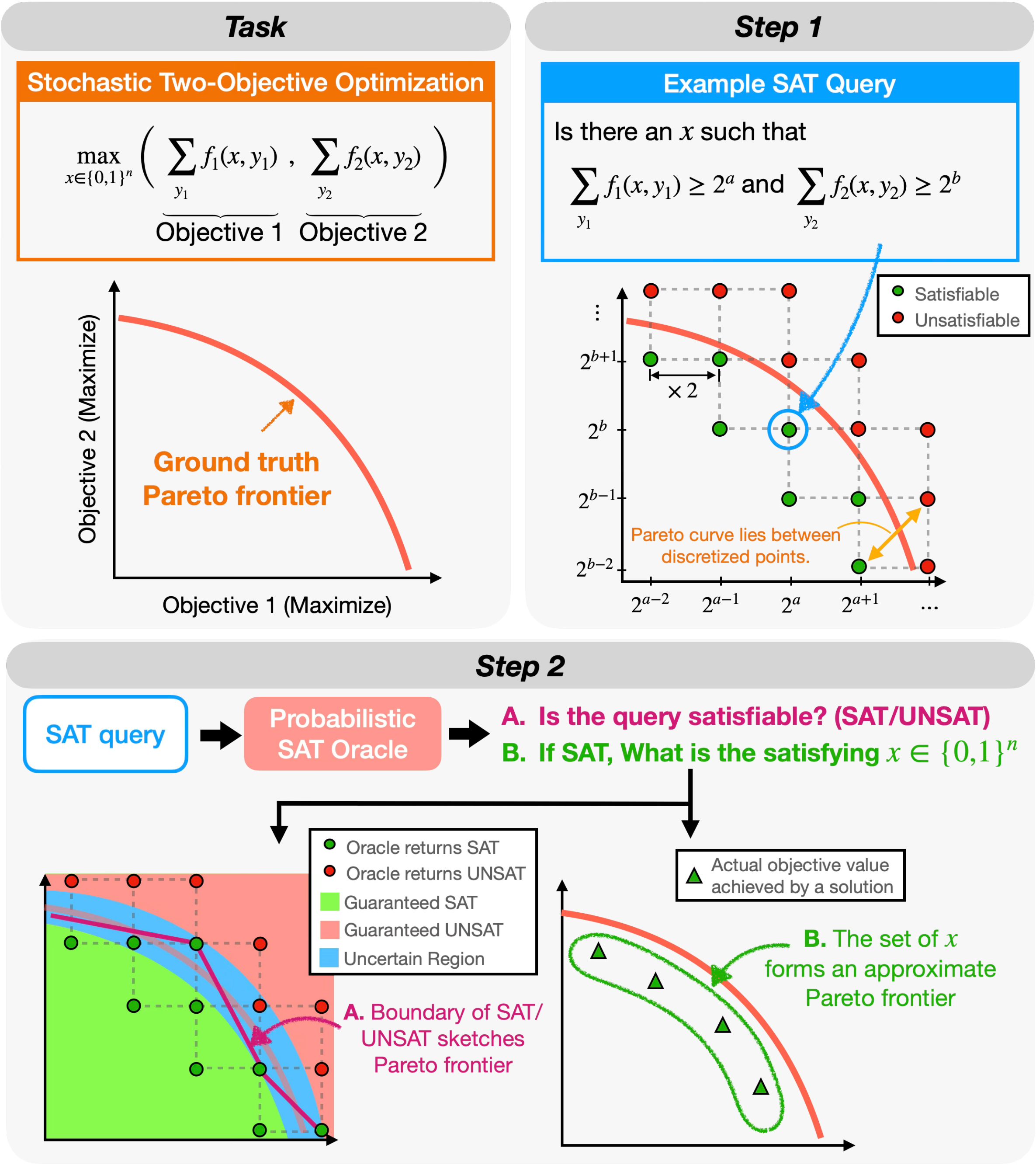}
    \caption{
    Illustration of our approach for solving SMOO problems in a two-objective maximization setting. 
    \textbf{Task:} Both objectives involve model counting, are defined over decision variables $x$ and latent variables $y_1, y_2$. The orange curve shows the Pareto frontier. 
    \textbf{Step~1:} By discretizing objective space multiplicatively by a factor of 2, optimization is converted into a set of SAT queries asking whether the thresholds at each grid point are jointly achievable, separating points into SAT (green) and UNSAT (red) regions. The true Pareto frontier is sandwiched between the adjacent green/red points. 
    \textbf{Step~2:} Since each SAT query is intractable, we instead use a probabilistic SAT oracle providing two guarantees: (1) correctness of the SAT/UNSAT outcome; and (2) if SAT, the returned solution achieves tightly guaranteed objective values. These guarantees yield a sketch of the Pareto frontier via the SAT/UNSAT boundary, and the corresponding solutions $x$ collectively form an approximate Pareto frontier.
}
    \label{fig:approach}
\end{figure}

We propose the \method algorithm for solving SMOO problems with provable guarantees. This section focuses on unweighted SMOO problems.  
Formally, each objective $\sum_{y_i} f_i(x, y_i)$ represents an unweighted model count, where $f_i : \{0,1\}^{n + |y_i|} \to \{0,1\}$.  Despite being unweighted, this setting already captures a broad class of stochastic objectives. 

Our \method algorithm (Algorithm \ref{alg:main}) returns a set of tuples. Each tuple is in the form of $(x, p)$, where $x$ is the solution, i.e., the value assignments to the binary decision variables, and $p$ is a vector of estimated objective values under the assignment $x$. We are able to obtain the  following two types of theoretical guarantees for the \method algorithm: 
\begin{enumerate}
    \item (\textbf{Quality of the Estimated Objective Values}) 
    The collection of the vectors of estimated objective values $p$ forms a high-quality sketch of the Pareto frontier (Figure~\ref{fig:approach}, Step~2,~A). At a high level, with high probability (e.g., 99\%), each estimated vector lies within a $2^{\epsilon}$ multiplicative distance from a vector made from true Pareto-optimal objective values and vice versa. 
    We restate the exact form of Theorem \ref{thm:pareto-curve} here for easy reference:
    
    \begin{itshape}
    ({Theorem~\ref{thm:pareto-curve}}) Fix an error bound $\delta \in (0,1)$ and an approximation factor $2^{\epsilon}$ with $\epsilon \geq 3$.  
Let $\mathcal{F}_{\mathrm{dom}} \subseteq \{0,1\}^n \times \mathbb{R}_{\ge 0}^k$ denote the set of tuples $(x,p)$ returned by Algorithm~\ref{alg:main}.  
Then, with probability at least $1-\delta$, the following holds:
\begin{itemize}
    \item For every Pareto-optimal solution $x_{\mathrm{opt}}$ with objective values
    $Q_i = \sum_{y_i} f_i(x_{\mathrm{opt}},y_i)$,
    for $i=1,\ldots,k$.  
    There exists $(x,p) \in \mathcal{F}_{\mathrm{dom}}$ with the estimated objective values $p=(p_1,\dots,p_k)$ such that
    $2^{\epsilon} p_i \ge Q_i, \forall i$.
    \item Conversely, for every tuple $(x,p)\in\mathcal{F}_{\mathrm{dom}}$ returned by \method, there exists a Pareto-optimal solution $x_{\mathrm{opt}}$, achieving objective values $Q_i = \sum_{y_i} f_i(x_{\mathrm{opt}},y_i)$,  for $i=1,\ldots,k$, such that 
    $2^{\epsilon} Q_i \ge p_i$, $\forall i$.
\end{itemize}
\end{itshape}
    \medskip
    \item (\textbf{Quality of Solutions}): 
    According to Definition \ref{def:approx_pareto}, the approximate Pareto frontier is the \textit{set of solutions} $x$.
    We are able to prove that 
    the set of solutions found by \method establishes a $2^{2\epsilon - 1}$-approximate Pareto frontier with high probability. 
    In other words, when the objective values at these solutions are evaluated exactly, they are under the true Pareto curve by at most a factor of $2^{2\epsilon - 1}$ (Figure~\ref{fig:approach}, Step~2,~B). 
    We again restate Theorem \ref{thm:pareto-solution} here:
    
    \begin{itshape}
        (Theorem \ref{thm:pareto-solution}): Fix an error bound $\delta \in (0,1)$ and  $\epsilon \geq 3$. 
    Let $\mathcal{F}_{\mathrm{dom}} \subseteq \{0,1\}^n \times \mathbb{R}_{\ge 0}^k$ denote the set of tuples $(x,p)$ returned by Algorithm~\ref{alg:main}. 
    Then, with probability at least $1-\delta$, the set of assignments $\{ x \in \{0,1\}^n : (x,p)\in\mathcal{F}_{\mathrm{dom}}\}$ constitutes a $2^{2\epsilon - 1}$-approximate Pareto frontier.
    \end{itshape}
\end{enumerate}

These guarantees together show that, even when all objective functions are computationally intractable, \method can (1) estimate Pareto-optimal objective values, providing a high-quality sketch of the true frontier within a $2^{\epsilon}$ multiplicative distance, and (2) produce solutions that form a $2^{2\epsilon - 1}$-approximate Pareto frontier when the objective values of the solutions are evaluated exactly.  

The proof of the aforementioned theoretic guarantees follows a multi-step process, which we will discuss the high-level idea below. Figure~\ref{fig:approach} also provides a graphical illustration.
\begin{itemize}
\item \textbf{Step 1: Approximating SMOO via Solving Discretized Decision Problems}: 
The first step is to approximate the SMOO problem by converting it into a set of decision problems. 
As illustrated in Figure~\ref{fig:approach} Step 1, we discretize the range of each objective using a grid of multiplicative scale. 

At each grid point, we formulate an SAT query: \emph{does there exist a solution $x$ such that every objective function at $x$ exceeds their respective threshold value defined by the grid}? 

Assuming that all SAT queries can be solved exactly by an oracle, the grid points will be separated into two parts -- the lower left part where the oracle returns SAT (denoted by the green points in Figure \ref{fig:approach} Step 1), and the upper right part where the oracle returns UNSAT (denoted by the red points in Figure \ref{fig:approach} Step 1). 
%
%
Intuitively, the true Pareto frontier is sandwiched between the satisfiable and unsatisfiable grid cells. Because every adjacent pair of green and red points is at most apart by a factor of 2, the true Pareto frontier, sandwiched between these pairs, should have a smaller distance, and hence less than a factor of 2, towards the top of the green points or the bottom of the red ones. 
Indeed, we can show in Lemma \ref{lem:discretize} that a 2-approximate Pareto frontier can be computed following a factor 2 multiplicative discretization. 

\medskip

\item \textbf{Step 2: Approximating Decision Problem Solutions Assuming a Probabilistic SAT Oracle Available}: 
Because each objective function in the SMOO problem is computationally intractable, the corresponding SAT queries cannot be solved exactly.  
Our theoretical guarantees will depend on having access to the following \emph{probabilistic SAT oracle}: 

\begin{itshape}
Given thresholds $(2^{l_1}, \dots, 2^{l_k})$ at a grid point, 
the oracle estimates whether there exists a solution $x$ such that all objective functions satisfy $\sum_{y_i} f_i(x, y_i) \ge 2^{l_i}$ for all $i$.  
It returns either $(\mathtt{True}, x^*)$, indicating that the thresholds are (approximately) jointly achievable at solution $x^*$, or $(\mathtt{False}, \bot)$, indicating that the thresholds are not achievable. Since the oracle is probabilistic, we require the following guarantees:
\begin{enumerate}
    \item \textbf{(Guaranteed UNSAT for high thresholds)}  
    If for all $x \in \{0,1\}^n$, at least for one $i \in \{1,\ldots,k\}$,
    $    \sum_{y_i} f_i(x, y_i) < 2^{l_i - l^*}$, 
    then the oracle returns $(\mathtt{False}, \bot)$ with probability at least $1-\eta$.
    
    \item \textbf{(Guaranteed SAT for low thresholds)}  
    If there exists $x \in \{0,1\}^n$ such that
    $
    \sum_{y_i} f_i(x, y_i) \ge 2^{l_i + l^*}, \quad \forall i = 1,\ldots,k$,    
    then the oracle returns $(\mathtt{True}, x^*)$ for some $x^* \in \{0,1\}^n$ satisfying
    $
    \sum_{y_i} f_i(x^*, y_i) \ge 2^{l_i - l^*}, \quad \forall i = 1,\ldots,k.
    $
    with probability at least $1-\eta$.
    
    \item \textbf{(Intermediate case)} Otherwise, with probability at least $1 - \eta$, the oracle returns either $(\texttt{False}, \bot)$ or $(\mathtt{True}, x^*)$. When it returns  $(\mathtt{True}, x^*)$, we need 
    $
    \sum_{y_i} f_i(x^*, y_i) \ge 2^{l_i - l^*} \quad \text{for all } i \in \{1,\ldots,k\}.
    $
\end{enumerate}
\end{itshape}


The exact definition of the probabilistic SAT oracle is in Definition \ref{def:sat-oracle}. 

The probabilistic oracle makes the analysis more interesting when we return to the graphical illustration. As shown in Figure~\ref{fig:approach}, Step~2 A, a third, intermediate region emerges between the SAT and UNSAT regions where the probabilistic SAT oracle is \emph{uncertain} (intermediate case). In this region, the oracle cannot determine whether all thresholds are achievable, subsuming a fixed multiplicative slack $2^{l^*}$. Inside the intermediate region, it may return either $(\mathtt{False}, \bot)$ or $(\mathtt{True}, x^*)$ with a candidate solution $x^*$. 


Although the presence of the third intermediate region complicates the analysis, we can show that
(1) the set of Pareto non-dominated grid points at which the oracle returns \texttt{True}
(i.e., the upper-right-most green points in Figure~1, Step~2A) sketches the Pareto frontier curve (Theorem~\ref{thm:pareto-curve}), and
(2) the corresponding solutions $x$ form an approximate Pareto frontier. In other words, if we evaluate the objective values of those $x$ exactly, these values will be near Pareto-optimal (Theorem~\ref{thm:pareto-solution}). 

The proof of Theorem~\ref{thm:pareto-curve} is based on the fact that the true Pareto frontier must lie entirely within the intermediate uncertain region. 
Our analysis assumes that the SAT/UNSAT statuses reported by the probabilistic SAT oracle at all grid points are correct (i.e., they fall within the probability $1-\eta$). %
A union bound can be used to ensure that such a probability is big enough.
%
In this scenario, any point below the lower boundary of the uncertain region would be declared SAT by the oracle. This ensures that the true Pareto frontier is above the lower boundary. 
Contrarily, any point above the upper boundary would be declared UNSAT, hence ensuring that the true frontier is below the upper boundary. 
Because the width of the uncertain region is controllable, the upper-right-most green points in (Figure~\ref{fig:approach}, Step~2A) provide a faithful sketch of the Pareto frontier curve.


%

The estimated objective values at the approximate Pareto frontier obtained from Step 2A may not be achievable. This is because the probabilistic SAT oracle may return a solution $x^*$ that achieves discounted objective values. 
However, because the discount is at most $2^{l^*}$, these solutions approximate the true Pareto frontier well (Figure~\ref{fig:approach}, Step~2B), even when we evaluate their objective values exactly. 
They  collectively form an approximate Pareto frontier (Theorem~\ref{thm:pareto-solution}).

\medskip
%


\item {\textbf{Step 3: Probabilistic SAT Oracle Implementation.}}
In this step, we implement the probabilistic SAT oracle assumed in Step~2, thereby fulfilling the requirements of Theorems~\ref{thm:pareto-curve} and~\ref{thm:pareto-solution}.

Our implementation leverages approximate counting using hashing and randomization. 
As described in Section~\ref{sec:prelim-xor}, to determining whether a model count
$
\sum_{y \in \{0,1\}^{|y|}} f(x_0, y)
$
is greater than $2^l$, we can check the  satisfiability of the SAT formula
\[
    f(x_0, y) \wedge \mathtt{XOR}_1(y) \wedge \cdots \wedge \mathtt{XOR}_l(y),
\]
where each constraint $\mathtt{XOR}_i(y)$ is the logical XOR over a randomly sampled subset of variables in $y$.
Specifically, the SAT formula is \emph{likely to be satisfiable} when the model count $\sum_{y} f(x_0,y)$ exceeds $2^{l}$, and \emph{likely unsatisfiable} when it is below $2^{l}$. 
We can show that a single SAT query succeeds in estimating the model counting with constant probability.
By Lemma~\ref{lem:xor-counting}, this probability can be made strictly greater than $1/2$ with appropriate parameter choices.



We can amplify the oracle’s success probability using a majority-voting scheme.
Specifically, if a majority of multiple SAT instances with independently sampled XOR constraints are satisfiable (or unsatisfiable),
then the probability of correctly determining whether the model count is above (or below) the threshold can be made arbitrarily high.

In addition to estimating the model count for a fixed $x_0$, we must also identify such an assignment $x_0$ for which the model count exceeds the desired threshold.
The key observation is that any assignment $x$ with a very small value of $\sum_y f(x,y)$ has a small probability of satisfying the constructed SAT formula with XOR constraints. 
This probability is so low that, even after applying a union bound over exponentially many possible assignments $x$, the probability that any such ``bad'' assignment survives remains negligible. Thus, if the SAT formula is satisfiable for some assignment $x_0$, then with high probability the model count with $x_0$ achieves a substantial fraction of the target threshold.



\end{itemize}

\subsection{Step 1: Approximating SMOO via Solving Discretized Decision Problems}


\begin{wrapfigure}{r}{0.5\textwidth} 
    \centering
    \vspace{-2em} 
    \includegraphics[width=0.9\linewidth]{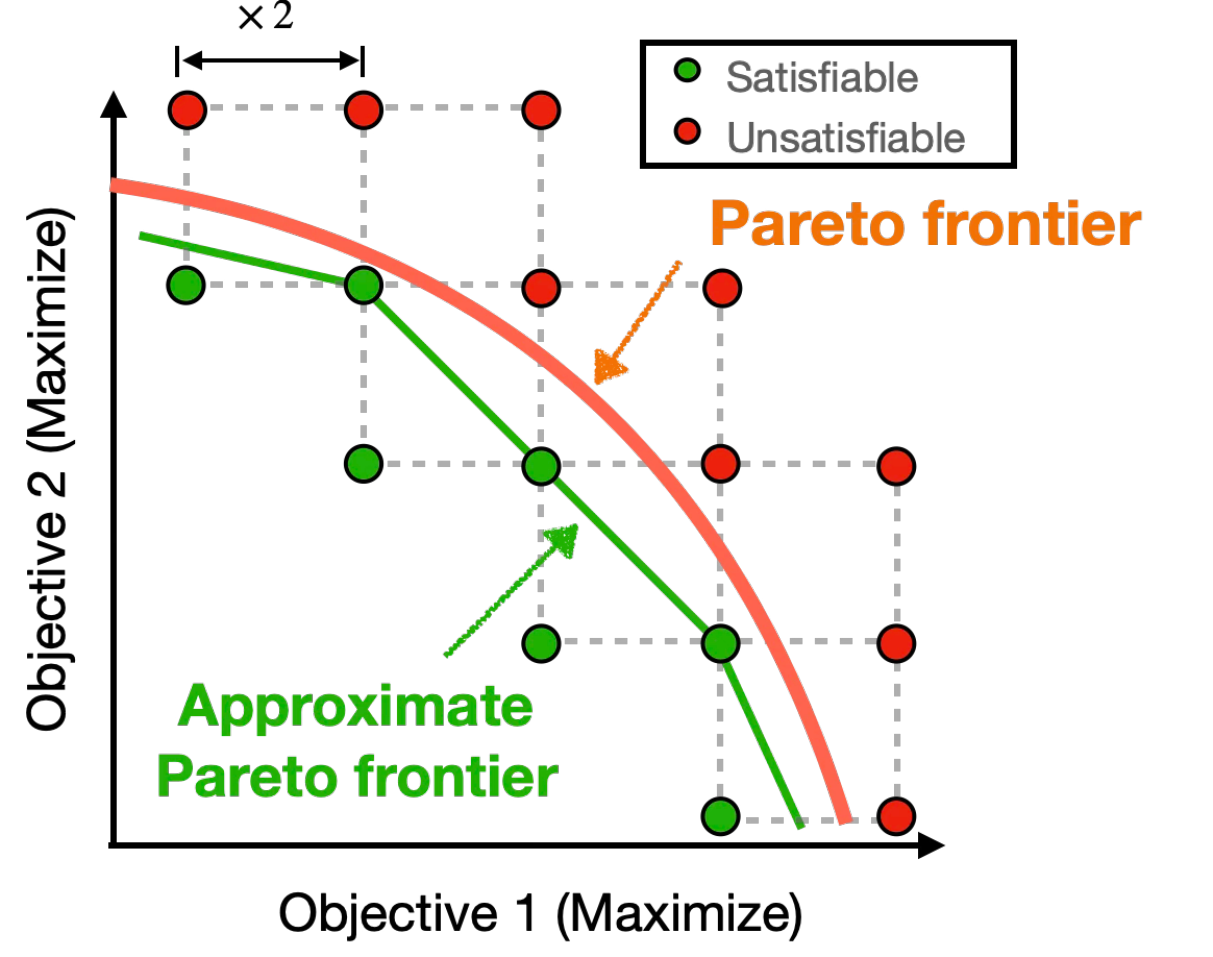}
    \caption{Example of solving a two-objective optimization problem via discretized decision problems, assuming exact inference were possible. The SAT boundary solutions would form a $2$-approximate Pareto frontier.}
    \label{fig:discretize}
    \vspace{-3em}
\end{wrapfigure}

This step transforms the original SMOO problem into a finite set of satisfiability queries.  
As illustrated in Figure~\ref{fig:approach}, Step~1, for the two-objective maximization case, the range of each objective function is discretized into a multiplicative-scale grid of threshold values (e.g., powers of two, where each grid point corresponds to $2, 2^2, \dots$).  
At each grid point, we query whether there exists a feasible solution that simultaneously satisfies all objectives at the corresponding threshold values.  
If every SAT query can be answered exactly, we can extract a $2$-approximate Pareto frontier, as shown in Figure~\ref{fig:discretize}. Intuitively, the solutions lying on the discretized SAT–UNSAT boundary achieve at least one half of the objective values of the Pareto-optimal solutions. A formal lemma is given below.

\begin{lemma}
\label{lem:discretize}
For the SMOO problem defined in Equation~\eqref{eq:smoop}, let
\[
\mathcal{P} = \left\{ \left(2^{l_1}, \ldots, 2^{l_k}\right) \middle| 0 \le l_i \le |y_i|, l_i \in \mathbb{Z}, \forall i \in \{1, \ldots, k\} \right\}.
\]
Suppose we have an exact SAT oracle that determines whether there exists an $x \in \{0,1\}^n$ such that
\begin{align}
    \Big( \sum_{y_1} f_1(x, y_1) \ge Q_1 \Big) \wedge \dots \wedge \Big( \sum_{y_k} f_k(x, y_k) \ge Q_k \Big),
    \label{eq:smoop-smc}
\end{align}
where $(Q_1, \dots, Q_k) \in \mathcal{P}$.  
Extract those $x^*$ such that:
\begin{itemize}
    \item Equation~\eqref{eq:smoop-smc} is satisfiable for $x^*$ with one threshold $(Q_1^*, \dots, Q_k^*)\in \mathcal{P}$.
    \item For every $(Q_1', \dots, Q_k') \in \mathcal{P}$ satisfying $Q_i' \ge Q_i^*$ for all $i$ and $Q_j' > Q_j^*$ for some $j$, Equation~\eqref{eq:smoop-smc} is unsatisfiable.
\end{itemize}
Then, the set of such $x^*$ establishes a $2$-approximate Pareto frontier.
\end{lemma}

\begin{proof}



For $i\in\{1,\dots,k\}$, define $F_i(x)$ to be $\sum_{y_i} f_i(x,y_i)$.
For any feasible $x$, define its rounded-down objective vector $\overline{F}(x)$ as
\[
\overline{F}(x)
 \coloneqq 
\bigl(2^{\lfloor \log_2 F_1(x)\rfloor},\dots,2^{\lfloor \log_2 F_k(x)\rfloor}\bigr), 
\]
in which $\overline{F}_i(x) = 2^{\lfloor \log_2 F_i(x)\rfloor}$. We can see that $\overline{F}(x) \in \mathcal P$. By construction, for every $i$,
\begin{equation}
\label{eq:rounding}
\overline{F}_i(x) \le F_i(x) < 2 \overline{F}_i(x).
\end{equation}

Fix an arbitrary Pareto-optimal solution $x_{\mathrm{opt}}$ of the original
SMOO problem.
Since $F_i(x_{\mathrm{opt}})\ge \overline{F}_i(x_{\mathrm{opt}})$ for all $i$,
the SAT formula \eqref{eq:smoop-smc} is satisfiable with threshold vector
$\overline{F}(x_{\mathrm{opt}}) \in \mathcal{P}$.

Among all threshold vectors in $\mathcal P$ for which
\eqref{eq:smoop-smc} is satisfiable, let $(Q_1^*,\dots,Q_k^*)$ be the special vector of thresholds defined in this Lemma~\ref{lem:discretize}, satisfying
\[
Q_i^* \geq \overline{F}_i(x_{\mathrm{opt}}),\quad i=1,\dots,k.
\]
Such a vector always exists: for example, it may be identical to
$\overline{F}(x_{\mathrm{opt}})$. Let $x^*$ be the corresponding assignment for this threshold vector $Q^*$.
By definition, $x^*$ is one of the extracted solutions described in the lemma.

From feasibility of $x^*$ we have
\[
F_i(x^*)  \ge  Q_i^*  \ge  \overline{F}_i(x_{\mathrm{opt}})
\quad \forall i.
\]
Combining this with \eqref{eq:rounding} yields
\[
F_i(x^*)  \ge  \overline{F}_i(x_{\mathrm{opt}})
 \ge  \tfrac{1}{2} F_i(x_{\mathrm{opt}})
\quad \forall i.
\]
This implies that,
\[
2 F(x^*)  \ge  F(x_{\mathrm{opt}}).
\]
In other words, the Pareto-optimal solution $x_{\mathrm{opt}}$ is dominated
by $x^*$ within a multiplicative factor of $2$.
Because the argument holds for every Pareto-optimal solution
$x_{\mathrm{opt}}$, the set of extracted solutions $\{x^*\}$ constitutes a
$2$-approximate Pareto frontier.
\end{proof}

In conclusion, with access to an exact SAT oracle, we can directly solve a series of SAT queries and extract a $2$-approximate Pareto frontier (or an even tighter approximation if finer discretization grids are used). However, solving each SAT query exactly is often highly intractable.

\subsection{Step 2: Approximating Decision Problem Solutions Assuming a Probabilistic SAT Oracle}
We introduce a probabilistic SAT oracle used to check whether there exists an assignment $x \in \{0,1\}^n$ such that all $k$ objectives achieve the thresholds simultaneously:
\[
\sum_{y_i} f_i(x, y_i) \ge 2^{l_i}, \quad \text{for all } i = 1,\ldots,k.
\]
Here, each $f_i : \{0,1\}^{n + |y_i|} \to \{0,1\}$ is a Boolean function over decision variables $x$ and latent variables $y_i$, $2^{l_i}$ is a threshold value ($l_i \in \mathbb{Z}_{\geq 0}$), and the model counting term $\sum_{y_i} f_i(x, y_i)$ is the $i$-th objective function in Equation~\eqref{eq:smoop}.  

The probabilistic oracle approximately solves the above query with high probability and tolerates a controlled error gap. We formalize this as follows (details of the oracle implementation are provided in the next step in Section~\ref{sec:sat-oracle}):

\begin{definition}[Probabilistic SAT Oracle]
\label{def:sat-oracle}
Let $f_i : \{0,1\}^{n + |y_i|} \to \{0,1\}$ for $i = 1, \ldots, k$ be Boolean functions, and let $2^{l_1}, \ldots, 2^{l_k}$, where $l_1, \ldots, l_k \in \mathbb{Z}_{\ge 0}$, be the target thresholds.  
A probabilistic SAT oracle, denoted
\[
\texttt{SAT-Oracle}(f_1,\ldots,f_k; l_1,\ldots,l_k, l^*, \eta),
\]
takes as input an error gap parameter $l^* \ge 2$ and an error probability bound $\eta \in [0,1]$, and returns either $(\texttt{True}, x^*)$ with a solution $x^* \in \{0,1\}^n$ or $(\texttt{False}, \bot)$, satisfying the following:
\begin{enumerate}
    \item \textbf{(Guaranteed UNSAT for high thresholds)}  
    If for all $x \in \{0,1\}^n$,
    \[
    \sum_{y_i} f_i(x, y_i) < 2^{l_i - l^*} \quad \text{for some } i \in \{1,\ldots,k\},
    \]
    then the oracle returns $(\mathtt{False}, \bot)$ with probability at least $1-\eta$.
    
    \item \textbf{(Guaranteed SAT for low thresholds)}  
    If there exists $x \in \{0,1\}^n$ such that
    \[
    \sum_{y_i} f_i(x, y_i) \ge 2^{l_i + l^*}, \quad \forall i = 1,\ldots,k,
    \]
    then the oracle returns $(\mathtt{True}, x^*)$ for some $x^* \in \{0,1\}^n$ satisfying
    \[
    \sum_{y_i} f_i(x^*, y_i) \ge 2^{l_i - l^*}, \quad \forall i = 1,\ldots,k.
    \]
    with probability at least $1-\eta$.
    
    \item \textbf{(Intermediate case)} Otherwise, with probability at least $1 - \eta$, the oracle returns either $(\texttt{False}, \bot)$ or $(\mathtt{True}, x^*)$. When it returns  $(\mathtt{True}, x^*)$, we need 
    \[
    \sum_{y_i} f_i(x^*, y_i) \ge 2^{l_i - l^*} \quad \text{for all } i \in \{1,\ldots,k\}.
    \]
\end{enumerate}
\end{definition}

The \textsc{SAT} oracle in Definition~\ref{def:sat-oracle} serves as a probabilistic verifier for specific objective thresholds.  
Given candidate thresholds $2^{l_1},\ldots,2^{l_k}$, it determines whether there exists a decision $x$ that achieves sufficiently high model counts across all objectives.  
If such an $x$ exists, the oracle returns \texttt{True} along with an assignment $x^*$. Due to the probabilistic nature, we can only guarantee that $x^*$ meets slightly relaxed thresholds. If no $x$ satisfies even the relaxed thresholds, it returns \texttt{False} with high probability.  
The parameter $2^{l^*}$ introduces an intermediate \emph{uncertain} region between the high-confidence \texttt{True} and \texttt{False} regions. See Figure~\ref{fig:approach}~Step 2A for a graphical illustration. 


\begin{algorithm}[!h]
\caption{$\texttt{\method}(\sum f_1,\dots, \sum f_k, \delta, \epsilon)$}
\label{alg:main}
\begin{algorithmic}[1]
\Require Objective functions $\{\sum f_i \}_{i=1}^k$; error probability bound $\delta$; approximation factor $\epsilon$.

\State $\mathcal{P} \gets \left\{ \left(2^{l_1}, \ldots, 2^{l_k} \right) \middle| 0 \le l_i \le |y_i|, l_i \in \mathbb{Z}, \forall i \in \{1,\ldots,k\} \right\}$.
\State $l^*\gets \epsilon -1, \quad \eta \gets \delta / |\mathcal{P}|$.
\State $\mathcal{F} \gets \emptyset$.
\For{each point $p=(2^{l_1},\dots, 2^{l_k}) \in \mathcal{P}$}
    \State $(\text{isSat}, x^*) \gets \texttt{SAT-Oracle}(f_1,\ldots,f_k, l_1,\ldots,l_k, l^*, \eta)$.
    \If{$\text{isSat} = \texttt{True}$}
        \State $\mathcal{F} \gets \mathcal{F} \cup \{(x^*, p)\}$.
    \EndIf
\EndFor
\State $\mathcal{F}_{dom} = \left\{ (x, p) \in \mathcal{F} \middle| \nexists (x', p') \in \mathcal{F} \text{ with } p' \geq p ~\text{element-wise}\right\}$
\State \Return $\mathcal{F}_{dom}$.
\end{algorithmic}
\end{algorithm}

Our algorithm based on the probabilistic SAT oracle is presented in Algorithm~\ref{alg:main}. Our main contributions are: with high probability, (1) recovering a high-quality sketch of the Pareto curve (Theorem~\ref{thm:pareto-curve}), and (2) finding the approximate Pareto frontier, i.e., the set of approximately dominating solutions (Theorem~\ref{thm:pareto-solution}).

\begin{theorem}[High-quality Pareto frontier curve]
\label{thm:pareto-curve}
Fix an error bound $\delta \in (0,1)$ and an approximation factor $2^{\epsilon}$ with $\epsilon \geq 3$.  
Let $\mathcal{F}_{\mathrm{dom}} \subseteq \{0,1\}^n \times \mathbb{R}_{\ge 0}^k$ denote the set of tuples $(x,p)$ returned by Algorithm~\ref{alg:main}.  
Then, with probability at least $1-\delta$, the following holds:
\begin{itemize}
    \item For every Pareto-optimal solution $x_{\mathrm{opt}}$ with objective values
    $Q_i = \sum_{y_i} f_i(x_{\mathrm{opt}},y_i)$\footnote{We assume that $Q_i \ge 2^{\epsilon}$ for all $i$. If some $Q_i < 2^{\epsilon}$, the corresponding Pareto point lies close to $0$, where the multiplicative $2^{\epsilon}$-approximation guarantee cannot be established.}, for $i=1,\ldots,k$.  
    There exists $(x,p) \in \mathcal{F}_{\mathrm{dom}}$ with the estimated objective values $p=(p_1,\dots,p_k)$ such that
    $2^{\epsilon} p_i \ge Q_i, \forall i$.
    \item Conversely, for every tuple $(x,p)\in\mathcal{F}_{\mathrm{dom}}$ returned by \method, there exists a Pareto-optimal solution $x_{\mathrm{opt}}$, achieving objective values $Q_i = \sum_{y_i} f_i(x_{\mathrm{opt}},y_i)$,  for $i=1,\ldots,k$, such that 
    $2^{\epsilon} Q_i \ge p_i$, $\forall i$.
\end{itemize}

\end{theorem}

\begin{proof}
We analyze Algorithm~\ref{alg:main} using the SAT oracle specified in Definition~\ref{def:sat-oracle}.  
Define the grid
\[
\mathcal{P} = \{(2^{l_1},\ldots,2^{l_k}) \mid 0\le l_i\le |y_i|,~l_i\in\mathbb{Z},~\forall i\}.
\]
and $l^*=\epsilon-1$ and $\eta=\delta/|\mathcal{P}|$. For each grid point $p=(2^{l_1},\ldots,2^{l_k})\in\mathcal{P}$, one call
\[
(\text{isSat},x^*) \gets \texttt{SAT-Oracle}(f_1,\ldots,f_k; l_1,\ldots,l_k, l^*, \eta),
\]
satisfies the guarantee stated in Definition \ref{def:sat-oracle} with probability at least $1-\eta$.

Define probabilistic event~\eqref{eq:eventE} as 
one in which oracle calls at all the grid points $p\in\mathcal{P}$ satisfy Definition~\ref{def:sat-oracle}. 
By the union bound, the probability that event~\eqref{eq:eventE} happens is at least
\[
\Pr\big(\text{Event (E) happens}\big) \ge
1 - \sum_{p\in\mathcal{P}}\eta = 1-\delta.
\tag{E}
\label{eq:eventE}
\]

The following discussions are conditioned in probabilistic scenarios in which event~\eqref{eq:eventE} happens:

\begin{enumerate}
    \item[(a)]\emph{(Every Pareto-optimal solution is $2^\epsilon$-dominated by a solution in $\mathcal{F}_{dom}$):}
Fix any Pareto-optimal $x_{\mathrm{opt}}$ with corresponding objective values
\[
Q_i := \sum_{y_i} f_i(x_{\mathrm{opt}},y_i), \qquad i=1,\ldots,k,
\]
Define
\[
q_i:=\lfloor \log_2 Q_i\rfloor, \qquad
l_i := q_i-l^*,\qquad p_i:=2^{l_i},\qquad p=(p_1,\ldots,p_k).
\]
Then $2^{q_i} \le Q_i < 2^{q_i+1}$, hence
\[
2^{l_i+l^*} = 2^{q_i} \le Q_i < 2^{q_i+1} = 2^{l_i+l^*+1}.
\]
Since $Q_i \ge 2^{l_i+l^*}$ for all $i$, condition~\emph{Guaranteed SAT} in Definition~\ref{def:sat-oracle} is met with input $\{l_i\}_{i=1}^k$ and $l^*$, so the SAT oracle returns $(\texttt{True},x)$ for some $x$, and $(x,p)$ will be included in $\mathcal{F}$.  
From the inequality above and $l^*=\epsilon-1$, we obtain
\[
Q_i \le 2^{\epsilon} p_i, \quad \forall i.
\]
If $(x,p)$ is removed in the final pruning step (Algorithm~\ref{alg:main} line 10), then there exists $(x',p')\in\mathcal{F}_{\mathrm{dom}}$ with $p'\ge p$ element-wise, which still guarantees
\[
Q_i \le 2^{\epsilon} p'_i, \quad \forall i.
\]
\item[(b)] \emph{(Every point in $\mathcal{F}_{dom}$ is $2^\epsilon$-dominated by a Pareto-optimal solution):}
Fix $(x,p)\in\mathcal{F}_{\mathrm{dom}}$ with $p_i=2^{l_i}$.  
By the \emph{Guaranteed SAT} and \emph{Intermediate case} guarantee in Definition~\ref{def:sat-oracle} and conditioning on event~\eqref{eq:eventE}, we have
\[
\sum_{y_i} f_i(x,y_i) \ge 2^{l_i-l^*}, \quad \forall i.
\]
Thus there exists a Pareto-optimal solution $x_{\mathrm{opt}}$ with objectives $Q_i \ge 2^{l_i-l^*}$. Multiplying both sides by $2^{l^*+1}=2^\epsilon$ yields
\[
2^{\epsilon} Q_i \ge 2^{\epsilon}(2^{l_i-l^*}) = 2^{l_i+1} \ge p_i,
\]
as required.
\end{enumerate}
This proves both directions of the theorem under event~\eqref{eq:eventE}, which occurs with probability at least $1-\delta$.
\end{proof}

\begin{theorem}[Approximate Pareto frontier]
    \label{thm:pareto-solution}
    Fix an error bound $\delta \in (0,1)$ and an approximation factor $2^{\epsilon}$ with $\epsilon \geq 3$. 
    Let $\mathcal{F}_{\mathrm{dom}} \subseteq \{0,1\}^n \times \mathbb{R}_{\ge 0}^k$ denote the set of tuples $(x,p)$ returned by Algorithm~\ref{alg:main}. 
    Then, with probability at least $1-\delta$, the set of assignments $\{ x \in \{0,1\}^n : (x,p)\in\mathcal{F}_{\mathrm{dom}}\}$ constitutes a $2^{2\epsilon - 1}$-approximate Pareto frontier.
\end{theorem}

\begin{proof}
The proof also conditions on the event \eqref{eq:eventE}, which occurs with probability at least $1-\delta$. When this event occurs, the claims in Definition~\ref{def:sat-oracle} are assumed to hold for the probabilistic SAT oracle at all grid points $p \in \mathcal{P}$. We will prove that the returned set $\mathcal{F}_{\mathrm{dom}}$ can establish a $2^{2\epsilon - 1}$-approximate Pareto frontier.

Fix a Pareto-optimal solution $x_{\mathrm{opt}}$, and it achieves objectives:
\[
Q_i = \sum_{y_i} f_i(x_{\mathrm{opt}}, y_i), \qquad i=1,\dots, k
\]
and define 
\[
q_i:=\lfloor \log_2 Q_i\rfloor, \qquad
l_i := q_i-l^*,\qquad p_i:=2^{l_i},\qquad p=(p_1,\ldots,p_k).
\]
Then $2^{q_i} \le Q_i < 2^{q_i+1}$, hence
\[
2^{l_i+l^*} = 2^{q_i} \le Q_i < 2^{q_i+1} = 2^{l_i+l^*+1}.
\]
Since $Q_i \ge 2^{l_i+l^*}$ for all $i$, condition~\emph{Guaranteed SAT} in Definition~\ref{def:sat-oracle} is met with input $\{l_i\}_{i=1}^k$ and $l^*$, so the SAT oracle returns $(\texttt{True},x)$ for some $x$ satisfying
\[
\sum_{y_i} f_i(x,y_i) \ge 2^{l_i-l^*}, \quad \forall i.
\]
Relating this to $Q_i$, note that
\[
Q_i < 2^{l_i+l^*+1}
\Rightarrow
2^{l_i-l^*} \ge 2^{-2l^*-1}Q_i.
\]
Therefore
\[
\sum_{y_i} f_i(x,y_i) \ge 2^{-2l^*-1}Q_i, \quad \forall i.
\]
Or equivalently,
\[
2^{2l^*+1}\sum_{y_i} f_i(x,y_i) \ge Q_i, \quad \forall i.
\]
Therefore the Pareto-optimal solution $x_{\mathrm{opt}}$ is $2^{2l^*+1}$-dominated by $x$, and $(x,p)$ will be included in $\mathcal{F}$. 
However, $(x,p)$ might be excluded from the final output set $\mathcal{F}_{\text{dom}}$. If this happens, then there is instead a $(x', p') \in \mathcal{F}_{\text{dom}}$ where $p' \geq p$ element-wise, there is still the 
\begin{align*}
    &\sum_{y_i} f_i(x', y_i) \geq p_i' 2^{-l^*} \geq 2^{l_i - l^*} \geq 2^{-2l^*-1} Q_i \\
    &\Leftrightarrow 2^{2l^*+1} \sum_{y_i}  f_i(x', y_i) \geq Q_i,
    \qquad i=1,\dots,k
\end{align*}
so we can conclude each Pareto-optimal solution is $2^{2l^*+1}$-dominated by an $x$ where $(x, p) \in \mathcal{F}_{\text{dom}}$. 
Because $l^*=\epsilon-1$, the factor simplifies to $2^{2\epsilon-1}$.  
Thus each Pareto-optimal solution $x_{\mathrm{opt}}$ is $2^{2\epsilon-1}$-dominated by some $x$ returned by the oracle. Equivalently, the returned set of assignments $\{ x \in \{0,1\}^n : (x,p)\in\mathcal{F}_{\mathrm{dom}}\}$ constitutes a $2^{2\epsilon - 1}$-approximate Pareto frontier.    
\end{proof}

\subsection{Step 3: Probabilistic SAT Oracle Implementation}
\label{sec:sat-oracle}
In this section, we present the implementation of the probabilistic SAT oracle introduced in Definition~\ref{def:sat-oracle}. It consists of two main steps:
(1) \textbf{Amplifying XOR Counting Success Probability}: enhancing the XOR-based counting method described in Section~\ref{sec:prelim-xor} to estimate the model count $\sum_y f(x_0, y)$ with \emph{arbitrarily} high success probability; and
(2) \textbf{Implementing a SAT oracle}: going beyond merely model counting estimation for a fixed $x_0$: instead, we implement a SAT oracle that answers whether there exists an $x_0$ achieving a given target threshold and, if so, returns a corresponding satisfying assignment $x_0$.

\subsubsection{Amplifying XOR Counting Success Probability}

We begin by showing how to implement a more reliable model counter that amplifies the success probability of XOR counting (Section~\ref{sec:prelim-xor}). The key idea is as follows. Recall the basic \texttt{XOR-Counting} oracle (Algorithm~\ref{alg:xor-counting}), which constructs a Boolean formula with random XOR constraints whose satisfiability distinguishes between the cases
\[
\sum_{y} f(x_0,y) \ge 2^{\,l+l^*}
\quad \text{and} \quad
\sum_{y} f(x_0,y) \le 2^{\,l-l^*},
\]
with a constant error probability of
\(\tfrac{2^{l^*}}{(2^{l^*}-1)^2}\).
To reduce the error probability while keeping the uncertainty gap $l^*$ fixed, we apply a majority-voting scheme: we generate multiple independent Boolean formulas and aggregate them using a majority vote (Algorithm~\ref{alg:amp-xor-sat}). The satisfiability result of this aggregated formula estimates the model count with high probability; moreover, the probability of correctness increases with the number of voters (Lemma~\ref{lem:amp}).

For the detailed implementation, let
\[
\tau = \ln\left(\tfrac{1}{\eta}\right),
\]
which we refer to as the \emph{confidence parameter}. All logarithmic dependencies on the target error probability $\eta$ will be expressed in terms of $\tau$, which highlights that amplification only incurs a logarithmic overhead.

The complete algorithm is shown in Algorithm~\ref{alg:amp-xor-sat}, and Lemma~\ref{lem:amp} is the formal justification. 

\begin{algorithm}[!h]
\caption{$\texttt{Amplified-XOR-Counting}(f,l,l^*,\tau)$}
\label{alg:amp-xor-sat}
\begin{algorithmic}
\Require Boolean formula $f$; number of XOR constraints $l$; error gap $l^*$, confidence parameter $\tau$
\State $m \gets \left \lceil \frac{2p}{(p-\frac{1}{2})^2} \tau \right \rceil$, where $p = 1 - \frac{2^{l^*}}{(2^{l^*}-1)^2}$
\For{$i = 1,\ldots,m$}
    \State Randomly sample $\mathtt{XOR}_1(y^{(i)}), \ldots, \mathtt{XOR}_l(y^{(i)})$
    \State $\psi_i(x, y^{(i)}) \gets f(x, y^{(i)}) \land \mathtt{XOR}_1(y^{(i)}) \land \dots \land \mathtt{XOR}_l(y^{(i)})$
\EndFor
\State $\Psi(x, y^{(1)}, \dots, y^{(m)}) \gets \mathtt{Majority} (\psi_1,\dots, \psi_m)$
\State \Return $\Psi$
\end{algorithmic}
\end{algorithm}

\begin{lemma}[XOR Counting Probability Amplification]
\label{lem:amp}
Let $f(x, y)$ be a Boolean function, and let $l, l^* \in \mathbb{Z}_{\ge 0}$ with $l^* \ge 2$.  
For any error probability $\eta \in (0,1)$, let 
\[
\Psi(x, y^{(1)}, \dots, y^{(m)}) \gets \texttt{Amplified-XOR-Counting}(f, l, l^*, \tau),
\quad \text{where } \tau = \ln \tfrac{1}{\eta}.
\]
Fix an assignment $x$ at $x_0 \in \{0,1\}^n$. Then, with probability at least $1 - \eta$, the following holds:
\begin{itemize}[align=left, leftmargin=0pt, labelwidth=0pt, itemindent=!]
    \item If $\sum_{y} f(x_0, y) \ge 2^{l + l^*}$, then $\Psi(x_0, y^{(1)}, \dots, y^{(m)})$ is satisfiable for some $y^{(1)}, \dots, y^{(m)}$.
    \item If $\sum_{y} f(x_0, y) \le 2^{l - l^*}$, then $\Psi(x_0, y^{(1)}, \dots, y^{(m)})$ is unsatisfiable for all $y^{(1)}, \dots, y^{(m)}$.
\end{itemize}
\end{lemma}

\begin{proof}
The main idea is that Algorithm~\ref{alg:amp-xor-sat} produces a Boolean formula $\Psi(x, y^{(1)}, \dots, y^{(m)})$, which takes the majority vote of $m \in \mathbb{Z}_{\ge 0}$ Boolean formulae $\psi_i(x, y^{(i)})$ for $i = 1, \dots, m$ independently generated by Algorithm~\ref{alg:xor-counting}.    
By aggregating the satisfiability of each $\psi_i$ through a majority vote, the error probability can be reduced arbitrarily. The formal proof is as follows.

Suppose the fixed $x_0$ satisfies $\sum_{y} f(x_0, y) \ge 2^{l + l^*}$, then by Lemma~\ref{lem:xor-counting}, $\psi_i(x_0,y^{(i)})$ is satisfiable for some $y^{(i)}$ with a probability at least
\[
p = 1 - \frac{2^{l^*}}{(2^{l^*}-1)^2}.
\]  
Since $l^* \geq 2$, the probability $p > 1/2$. The probability of $\Psi(x_0, y^{(1)}, \dots, y^{(m)})$ is satisfiable for some $y^{(1)}, \dots, y^{(m)}$ can be bounded by the Chernoff bound: 
the probability that fewer than half of the $\{\psi_1,\dots, \psi_m\}$ are correct is bounded as
\begin{align*}
& \Pr\left( \text{$\Psi(x_0, y^{(1)}, \dots, y^{(m)})$ is satisfiable}\right) \\
&= \Pr\left( \text{The majority of $\{\psi_i(x_0,y^{(i)})\}_{i=1}^m$ are satisfiable for some $y^{(i)}$} \right)\\
&=\Pr\left(\sum_{i=1}^m I(\text{$\psi_i(x_0,y^{(i)})$ is satisfiable for some $y^{(i)}$}) > \frac{m}{2}\right) \\
&\geq 1 - \exp \left(-\tfrac{(p-\frac{1}{2})^2}{2p} m\right).
\end{align*}
Choosing 
\[
m \geq \frac{2p}{(p-\tfrac{1}{2})^2} \tau, \mbox{ where } \tau = \ln\tfrac{1}{\eta},
\]
ensures the probability is at least $1 - \eta$. 

Similarly, if the fixed $x_0$ satisfies $\sum_{y} f(x_0, y) \leq 2^{l - l^*}$, $\Psi(x_0, y^{(1)}, \dots, y^{(m)})$ is unsatisfiable for all $y^{(1)}, \dots, y^{(m)}$ with probability at least $1-\eta$.
\end{proof}

This amplification scheme reduces the error probability to any desired value $\eta$, while requiring a majority vote among $O(\tau)$ Boolean SAT problems, where $\tau = \ln(1/\eta)$.
\texttt{Amplified-XOR-Counting} requires implementing a majority operator.
We implement it using auxiliary indicator variables and a linear cardinality constraint.
Given Boolean formulas $\{\psi_1,\dots,\psi_m\}$, we introduce binary variables $b_i \in \{0,1\}$ such that $b_i = 1$ if and only if $\psi_i$ is satisfied.
This relationship is enforced via biconditional constraints $b_i \Leftrightarrow \psi_i$.
The majority condition is expressed as
$
\sum_{i=1}^m b_i \ge \lceil m/2 \rceil,
$
which requires at least half of the formulas to be satisfiable.
We encode all constraints using mixed-integer programming (MIP). 

Implementing majority logic in satisfiability has been extensively studied. Prior work proposes native majority-logic encodings within propositional logic \cite{pacuit2004majority,amaru2014majority}, as well as SAT solvers specialized for majority logic \cite{chou2016majorsat}.
While our MIP-style approach does not aim to outperform these Boolean encodings, it provides a simple and modular implementation that integrates naturally with existing frameworks (e.g., \textsc{CPLEX}).

\subsubsection{\texttt{XOR-SAT} Oracle}
The amplified XOR-counting oracle above generates a Boolean formula whose satisfiability can be used to estimate whether the model count is above or below a given threshold by an uncertainty margin, i.e., whether $\sum_y f(x_0, y) \ge 2^{l+l^*}$ or $\sum_y f(x_0, y) \le  2^{l-l^*}$ for a fixed $x_0$.

We now utilize this tool to implement an SAT oracle that answers a stronger query: does there exist such an $x_0$ that simultaneously satisfies all objective thresholds?  
Formally, the SAT oracle must handle multiple model-counting terms to accommodate multiple objectives.  
For $k$ objective functions defined over the model counts of Boolean functions $f_1, \dots, f_k$, with thresholds $2^{l_1}, \dots, 2^{l_k}$ and approximation gap $2^{l^*}$, we aim to distinguish between
\[
\exists x, \forall i: \sum_{y_i} f_i(x, y_i) \ge 2^{l_i + l^*}
\qquad\text{vs.}\qquad
\forall x, \exists i: \sum_{y_i} f_i(x, y_i) \le 2^{l_i - l^*}.
\]

In the SMOO setting, each objective $\sum f_i$ corresponds to its own model-counting problem.  
The oracle checks whether there exists a decision $x$ that simultaneously achieves sufficiently large counts ($2^{l_i}$) across all $k$ objectives.  
Intuitively, the first case asserts the existence of a ``universally strong'' decision $x$ that satisfies the higher thresholds $2^{l_i + l^*}$, while the second case states that no such decision exists: for every one $x$, there must exist one objective $f_i$, such that the model count $\sum_{y_i} f_i(x, y_i) \le 2^{l_i - l^*}$. 
The gap $2^{l^*}$ provides a buffer between these two regimes, preventing ambiguity due to the randomness of XOR counting. The oracle, \texttt{XOR-SAT}, is implemented in Algorithm~\ref{alg:xor-sat}. It returns either $(\mathtt{True}, x^*)$, indicating that there exists an assignment $x^*$ whose objective values ``approximately'' meet all specified thresholds, or $(\mathtt{False}, \bot)$, indicating that no such assignment $x^*$ exists. The formal proof is in Theorem~\ref{thm:xor-sat}.

\begin{algorithm}[!h]
\caption{$\texttt{XOR-SAT}(f_1,\dots,f_k,l_1,\dots,l_k, l^*,\eta)$}
\label{alg:xor-sat}
\begin{algorithmic}[1]
\Require Objectives $\{f_i\}_{i=1}^k$; thresholds $\{l_i\}_{i=1}^k$; gap $l^*$; target error $\eta$
\State $\tau \gets \max\{\ln k,n\ln 2\}+\ln \frac{2}{\eta}$
\For{$i=1,\ldots,k$}
  \State $\Psi_i(x, y_i^{(1)}, \dots, y_i^{(m)}) \gets \texttt{Amplified-XOR-Counting}(f_i, l_i, l^*, \tau)$
\EndFor
\If{ $\exists~(x^*, \{y_i^{(1)}\}_{i=1}^k,\dots, \{y_i^{(m)}\}_{i=1}^k)$ s.t. $\bigwedge_{i=1}^k \Psi_i(x^*, y_i^{(1)}, \dots, y_i^{(m)})=\mathtt{True}$}
  \State \Return $(\mathtt{True}, x^*)$
\Else
  \State \Return $(\mathtt{False}, \bot)$
\EndIf
\end{algorithmic}
\end{algorithm}

\begin{lemma}[\texttt{XOR-SAT} Solution Quality Guarantee]
\label{lem:quality}
Let $l^* \ge 2$ and $0 < \eta < 1$.  
Run $\texttt{XOR-SAT}(\{f_i\}_{i=1}^k, \{l_i\}_{i=1}^k, l^*, \eta)$, which returns either $(\mathtt{True}, x^*)$ or $(\mathtt{False}, \bot)$.  
Then, with probability at least $1 - \tfrac{\eta}{2}$, the oracle returns either $(\mathtt{False}, \bot)$ or $(\mathtt{True}, x^*)$. When it returns $(\mathtt{True}, x^*)$, it must satisfy:
\[
\sum_{y_i} f_i(x^*, y_i) \geq 2^{l_i - l^*}, \quad \forall i = 1,\ldots,k.
\]
\end{lemma}

\begin{proof}[Proof of Lemma~\ref{lem:quality}.]
We will prove an equivalent claim: with probability at least $1 - \tfrac{\eta}{2}$, the oracle \textbf{does not return} any $(\mathtt{True}, x^*)$ such that
\[
\sum_{y_i} f_i(x^*, y_i) < 2^{l_i - l^*} \quad \text{for some } i \in \{1, \ldots, k\}.
\]
In other words, any $x^*$ that violates at least one threshold $2^{l_i}$ by a margin of $2^{l^*}$ will not be returned. Let's denote the set of such $x^*$ as
\[
\mathcal{X}^- = \Bigl\{x \in \{0,1\}^n : \exists i \text{ such that } \sum_{y_i} f_i(x,y_i) < 2^{l_i-l^*}\Bigr\}.
\]

Algorithm~\ref{alg:xor-sat} sets
\[
\tau \gets \max\{\ln k, n \ln 2\} + \ln \tfrac{2}{\eta},
\qquad
\eta' = e^{-\tau} \le \min\!\left\{\tfrac{\eta}{2k}, \tfrac{\eta}{2^{n+1}}\right\}.
\]
 
For any fixed $x^- \in \mathcal{X}^-$, let $i^\star$ be an index such that
\[
\sum_{y_{i^\star}} f_{i^\star}(x^-, y_{i^\star}) < 2^{l_{i^\star}-l^*}.
\]
Then the probability of returning $x^-$ is bounded by
\begin{align*}
& \Pr\bigl(\texttt{XOR-SAT returns }(\mathtt{True},x^-)\bigr) \\
&\le \Pr\Bigl(\Psi_{i^\star}(x^-,y_{i^\star}^{(1)}, \dots, y_{i^\star}^{(m)}) \text{~is satisfiable for some~} (y_{i^\star}^{(1)}, \dots, y_{i^\star}^{(m)}) \Bigr) \\
&\le \eta'.    
\end{align*}

By the union bound, the probability of returning any $x \in \mathcal{X}^-$ is at most
\begin{align*}
&\Pr\Big(\exists x\in \mathcal{X}^-,~\text{\texttt{XOR-SAT} returns $(\texttt{True},x)$}\Big) \\
& = \Pr\Big( \bigcup_{x\in \mathcal{X}^-} \text{\texttt{XOR-SAT} returns $(\texttt{True},x)$} \Big) \\
&\le \sum_{x\in\mathcal{X}^-} \Pr\Big(\text{\texttt{XOR-SAT} returns $(\texttt{True},x)$} \Big) \\
& \le \sum_{x\in\{0,1\}^n} \eta' = 2^n \eta' \le \frac{\eta}{2},
\end{align*}
Thus, with probability at least $1 - \tfrac{\eta}{2}$, \texttt{XOR-SAT} does not return any $(\mathtt{True}, x)$ where $x \in \mathcal{X}^-$. Equivalently, \texttt{XOR-SAT} returns either $(\mathtt{False}, \bot)$ or $(\mathtt{True}, x^*)$ where
\[
\sum_{y_i} f_i(x^*, y_i) \geq 2^{l_i - l^*}, \quad \forall i = 1,\ldots,k.
\]
\end{proof}

\begin{theorem}[\texttt{XOR-SAT} Oracle Properties]
\label{thm:xor-sat}
Let $l^* \ge 2$ and $0 < \eta < 1$.  
Run $\texttt{XOR-SAT}(\{f_i\}_{i=1}^k, \{l_i\}_{i=1}^k, l^*, \eta)$, which returns either $(\mathtt{True}, x^*)$ or $(\mathtt{False}, \bot)$.  
Then the oracle satisfies the following properties:
\begin{enumerate}
    \item \textbf{(Guaranteed UNSAT for high thresholds)}  
    If for all $x \in \{0,1\}^n$,
    \[
    \sum_{y_i} f_i(x, y_i) < 2^{l_i - l^*} \quad \text{for some } i \in \{1,\ldots,k\},
    \]
    then the oracle returns $(\mathtt{False}, \bot)$ with probability at least $1-\eta$.
    
    \item \textbf{(Guaranteed SAT for low thresholds)}  
    If there exists $x \in \{0,1\}^n$ such that
    \[
    \sum_{y_i} f_i(x, y_i) \ge 2^{l_i + l^*}, \quad \forall i = 1,\ldots,k,
    \]
    then the oracle returns $(\mathtt{True}, x^*)$ for some $x^* \in \{0,1\}^n$ satisfying
    \[
    \sum_{y_i} f_i(x^*, y_i) \ge 2^{l_i - l^*}, \quad \forall i = 1,\ldots,k.
    \]
    with probability at least $1-\eta$.
    
    \item \textbf{(Intermediate case)} Otherwise, with probability at least $1 - \eta$, the oracle returns either $(\texttt{False}, \bot)$ or $(\mathtt{True}, x^*)$ such that
    \[
    \sum_{y_i} f_i(x^*, y_i) \ge 2^{l_i - l^*} \quad \text{for all } i \in \{1,\ldots,k\}.
    \]
\end{enumerate}
\end{theorem}

\begin{proof}[Proof of Theorem~\ref{thm:xor-sat}.]
Algorithm~\ref{alg:xor-sat} sets
\[
\tau \gets \max\{\ln k, n \ln 2\} + \ln \tfrac{2}{\eta},
\qquad
\eta' = e^{-\tau} \le \min\!\left\{\tfrac{\eta}{2k}, \tfrac{\eta}{2^{n+1}}\right\}.
\]

\noindent\textbf{(Guaranteed UNSAT for high thresholds).}  
If for all $x \in \{0,1\}^n$ we have
\[
\sum_{y_i} f_i(x, y_i) < 2^{l_i - l^*} \quad \text{for some } i,
\]
then, by Lemma~\ref{lem:quality}, \texttt{XOR-SAT} returns either $(\mathtt{False}, \bot)$ or $(\mathtt{True}, x^*)$ with probability at least $1 - \tfrac{\eta}{2} > 1 - \eta$, where any returned $x^*$ satisfies
\[
\sum_{y_i} f_i(x^*, y_i) \ge 2^{l_i - l^*}, \quad \forall i.
\]
Since no such $x^*$ exists, the oracle will instead return $(\mathtt{False}, \bot)$ with probability at least $1 - \eta$.

\medskip
\noindent\textbf{(Guaranteed SAT for low thresholds).}  
Suppose there exists $x_0 \in \{0,1\}^n$ such that
\[
\sum_{y_i} f_i(x_0, y_i) \ge 2^{l_i + l^*}, \quad \forall i.
\]
Thus, $x_0$ achieves all thresholds $2^{l_i}$ with a strong margin.  
We prove the claim by showing that the probability of its \emph{compliment event}, either returning $(\mathtt{False}, \bot)$ or returning $(\mathtt{True}, x)$ with $\sum_{y_i} f_i(x,y_i) < 2^{l_i-l^*}$ for some $i$, is at most $\eta$.

\begin{enumerate}
    \item \emph{(Returning \texttt{False})} \texttt{XOR-SAT} returning \texttt{False} requires that at least $x_0$ is not returned. Hence, we will examine the probability that $x_0$ is not returned as an upper bound of returning \texttt{False}.
    By Lemma~\ref{lem:amp},
    \[
    \Pr\Bigl( \Psi_i(x_0, y_i^{(1)},\dots, y_i^{(m)}) \text{~is unsatisfiable for all~}(y_i^{(1)},\dots, y_i^{(m)}) \Bigr) \le \eta', \quad \forall i.
    \]
    Hence
    \begin{align*}
        &\Pr(\text{\texttt{XOR-SAT} returns \texttt{False}}) \\
        \leq & \Pr\Big(\text{$x_0$ is not returned by \texttt{XOR-SAT}}\Big) \\
        \leq & \Pr\Big(\bigcup_{i=1}^k \Psi_i(x_0, y_i^{(1)},\dots, y_i^{(m)}) \text{~is unsatisfiable}\Big) \\
        \leq & \sum_{i=1}^k \Pr\Big(\Psi_i(x_0, y_i^{(1)},\dots, y_i^{(m)}) \text{~is unsatisfiable}\Big) \\
        \leq & k \eta' \leq \tfrac{\eta}{2}.
    \end{align*}

    \item \emph{(Returning an undesired $x$)}  
    By Lemma~\ref{lem:quality}, the probability of returning any $(\mathtt{True}, x)$ with
    \[
    \sum_{y_i} f_i(x, y_i) < 2^{l_i - l^*} \quad \text{for some } i
    \]
    is at most $\tfrac{\eta}{2}$.
\end{enumerate}

Combining these bounds, let  
$A := \{\texttt{XOR-SAT returns } \mathtt{False}\}$ and  
$B := \{\texttt{XOR-SAT returns }(\mathtt{True},x) \text{ with } \sum_{y_i} f_i(x,y_i) < 2^{l_i - l^*} \text{ for some } i\}$.  
We have $\Pr(A) \le \tfrac{\eta}{2}$ and $\Pr(B) \le \tfrac{\eta}{2}$.  
Therefore,
\[
\Pr(\neg A \wedge \neg B) \ge 1 - \Pr(A) - \Pr(B) \ge 1-\eta.
\]
On the event that neither $A$ nor $B$ happens, the oracle returns $(\mathtt{True}, x^*)$ with
\[
\sum_{y_i} f_i(x^*,y_i) \ge 2^{l_i-l^*}, \quad \forall i.
\]

\medskip
\noindent\textbf{(Intermediate case).}  
If no $x$ achieves the stronger bound $\sum_{y_i} f_i(x,y_i) \ge 2^{l_i+l^*}$ for all $i$, but there exist $x$ such that
\[
\sum_{y_i} f_i(x,y_i) \ge 2^{l_i-l^*}, \quad \forall i,
\]
then only Lemma~\ref{lem:quality} applies, ensuring that with probability at least $1-\eta$, the oracle does not return any solution below the thresholds $2^{l_i-l^*}$.
\end{proof}

\subsection{{Sizes of SAT Queries Solved by \method}}
Consider an SMOO problem defined in Equation~\eqref{eq:smoop}, where the $k$ objective functions are unweighted counts of the form $\sum_{\mathbf{y}_i} f_i(\mathbf{x}, \mathbf{y}_i)$ for $i = 1, \dots, k$, the $n$ decision variables are $\mathbf{x} \in \{0,1\}^n$, and, for each $i$, $\mathbf{y}_i \in \{0,1\}^{|\mathbf{y}_i|}$ denotes the set of latent variables.

\method (Algorithm~\ref{alg:main}) solves this SMOO problem by encoding it into SAT queries. Each query asks whether the following formula is satisfiable:
\begin{align}
    \bigwedge_{i=1}^k \Psi_i(\mathbf{x}, \mathbf{y}_i^{(1)}, \dots, \mathbf{y}_i^{(m)})
    \label{eq:uw-smoop-sat}
\end{align}
where
\[
    \Psi_i(\mathbf{x}, \mathbf{y}_i^{(1)}, \dots, \mathbf{y}_i^{(m)})
    = \mathtt{Majority}(\psi_i^{(1)}, \dots, \psi_i^{(m)}).
\]

\noindent\textbf{Variables.}
Each query determines satisfiability over variables $\mathbf{x}$ and $\{\{\mathbf{y}_i^{(j)}\}_{i=1}^k\}_{j=1}^m$, where $j=1,\dots,m$ represents that latent variables $|\mathbf{y}_i|$ are copied $m$ times.
In total, each query involves
\[
n + m \sum_{i=1}^k |\mathbf{y}_i|
\]
Boolean variables.

\noindent\textbf{Constraints.}
Each SAT query is encoded using a MIP-style formulation involving Boolean variables and linear constraints.
It consists of the following formulas and  constraints:
\begin{itemize}[leftmargin=15pt]
    \item \emph{Formulas Encoding XOR Counting.}
    There are $k \times m$ formulas of the form
    \[
    \psi_i^{(j)}(\mathbf{x}, \mathbf{y}_i^{(j)})
    = f_i(\mathbf{x}, \mathbf{y}_i^{(j)})
    \land \mathtt{XOR}_1(\mathbf{y}_i^{(j)})
    \land \dots
    \land \mathtt{XOR}_\ell(\mathbf{y}_i^{(j)}),
    \quad i = 1, \dots, k,~ j = 1, \dots, m,
    \]
    where $k$ is the number of objectives, $\delta \in (0,1)$ is the user-specified error probability bound, $\epsilon \ge 3$ is the user-specified approximation gap, and
    \[
    m = \left\lceil \frac{2p}{\left(p-\frac{1}{2}\right)^2}\,\tau \right\rceil,
    \quad
    p = 1 - \frac{2^{\epsilon - 1}}{(2^{\epsilon - 1}-1)^2},
    \]
    \[
    \tau = \max\{\ln k,\, n \ln 2\}
    + \left(\sum_{i=1}^k \ln \left(|\mathbf{y}_i|\right) - \ln(\delta) + \ln 2\right).
    \]
    Here, $\mathtt{XOR}(\cdot)$ denotes an independent random XOR constraint. Each formula $\psi_i^{(j)}$ contains a number of XOR constraints ($\ell$) ranging from $0$ to $|\mathbf{y}_i|$.
    \item \emph{Constraints Encoding the SAT Query~\eqref{eq:uw-smoop-sat}.}
    We introduce auxiliary Boolean variables $b_i^{(j)}$ with constraints
    \[
    b_i^{(j)} \Leftrightarrow \psi_i^{(j)}(\mathbf{x}, \mathbf{y}_i^{(j)}), \quad i = 1, \dots, k,~ j = 1, \dots, m,
    \]
    and enforce majority constraints
    \[
    \sum_{j=1}^m b_i^{(j)} > \frac{m}{2}, \quad i = 1, \dots, k.
    \]
\end{itemize}

Assuming each unweighted objective function $f_i$ is encoded in SAT, we define the size of a SAT query relative to the size of the objective encoding. In a SAT query, there are $O(m)$ constraints, and each constraint has a size linear in the encoding of the objective function (with only a constant number of XOR constraints). Therefore, the size of a SAT query can be measured directly by the number of constraints. Given an approximation factor $\epsilon$, for a $2^\epsilon$-approximate Pareto frontier, let $|Y|=\max_i |\mathbf{y}_i|$. The resulting query size satisfies $O(m)=O(\tau)=O\left(n+\log\frac{1}{\delta}+k\log|Y|\right)$.

\textbf{Total number of SAT queries.}~~
The total number of SAT queries solved by \method is $\prod_{i=1}^k |\mathbf{y}_i|$.

\section{SMOO on Weighted Model Counting Objectives}
\label{sec:extension}

So far, we have assumed \emph{unweighted} model counting objectives, where $f_i:\{0,1\}^{n+|\mathbf{y}_i|}\to\{0,1\}$.  
In practice, however, many objectives are \emph{weighted} counts, where $f_i(\mathbf{x},\mathbf{y}_i)$ takes values in $\mathbb{R}_{\ge 0}$.  
Our framework extends to this case by reducing weighted counts to unweighted counts through an auxiliary construction. 

We develop w-\method which extends to weighted functions $f_i(\mathbf{x},\mathbf{y}_i)\in\mathbb{R}_{\ge 0}$. We are able to find Pareto frontiers with arbitrarily small approximation gaps.
The high-level idea is to construct a \emph{pseudo} SMOO problem whose objectives are unweighted sums that faithfully approximate the original weighted ones.  
With sufficient computation, the pseudo problem preserves the characteristics of the original objectives, and the approximation gap can be made arbitrarily small.

We introduce \texttt{w-\method} (Algorithm~\ref{alg:weighted-main}), which achieves controllable approximation precision for weighted counting objectives.  
The method trades computational budget for accuracy by moderately increasing the SAT query complexity, specifically, the number of variables and constraints.  

We provide two complementary results:  
(1) a theorem (Theorem \ref{thm:w-pareto-solution}) that characterizes the achievable approximation quality under a given computational budget, and  
(2) a corollary (Corollary \ref{cly:w-pareto-solution}) that specifies the computational requirements needed to attain a desired approximation factor.  
The results are stated as follows.

\begin{theorem}[Approximate Pareto Frontier under Limited Computational Budget]
\label{thm:w-pareto-solution}
Fix a probability $\delta \in (0,1)$ and a problem size factor $T \in \{1,2,\dots\}$.  
Given a discretization bit budget $b \in \{0,1,2,\dots\}$, define
\[
L_i \triangleq \min_{\mathbf{x},\mathbf{y}_i} f_i(\mathbf{x},\mathbf{y}_i), \quad
U_i \triangleq \max_{\mathbf{x},\mathbf{y}_i} f_i(\mathbf{x},\mathbf{y}_i), \quad 
\zeta(b) \triangleq \max_i \log_2 \Big(1 + \frac{U_i - L_i}{2^{\frac{5}{T}} L_i 2^b}\Big).
\]
Let $\mathcal{F}_{\mathrm{dom}} \subseteq \{0,1\}^n \times \mathbb{R}_{\ge 0}^k$ denote the set of tuples $(\mathbf{x},\mathbf{p})$ returned by  
$\texttt{w-\method}(\sum f_1, \dots, \sum f_k, \delta, T,b)$. Then, with probability at least $1-\delta$, the set of assignments 
\[
\bigl\{  \mathbf{x} \in \{0,1\}^n : (\mathbf{x},\mathbf{p}) \in \mathcal{F}_{\mathrm{dom}}  \bigr\}
\]
constitutes a $(2^{\frac{5}{T} + \zeta(b)})$-approximate Pareto frontier.  
\end{theorem}

\begin{corollary}[Approximate Pareto Frontier for a Target Approximation Factor]
\label{cly:w-pareto-solution}
Fix a probability $\delta \in (0,1)$ and a target approximation factor $\gamma > 1$. Define
\[
L_i \triangleq \min_{\mathbf{x},\mathbf{y}_i} f_i(\mathbf{x},\mathbf{y}_i), \quad
U_i \triangleq \max_{\mathbf{x},\mathbf{y}_i} f_i(\mathbf{x},\mathbf{y}_i).
\]
Run $\texttt{w-\method}(\sum f_1, \dots, \sum f_k, \delta, T, b)$, where 
\[
T = \frac{10}{\log_2(\gamma)} 
\quad \text{and} \quad 
b = \max_i \log_2 \Big(\frac{U_i}{L_i} - 1\Big) - \log_2(\gamma - \gamma^{\frac{1}{2}}),
\]
and let the returned set be $\mathcal{F}_{\mathrm{dom}}$.  
Then, with probability at least $1-\delta$, the set of assignments 
\[
\bigl\{  \mathbf{x} \in \{0,1\}^n : (\mathbf{x},\mathbf{p}) \in \mathcal{F}_{\mathrm{dom}}  \bigr\}
\]
constitutes a $\gamma$-approximate Pareto frontier.  
\end{corollary}


\subsection{Step 1: Reducing Weighted Model Counting to Unweighted Counting}

\begin{figure}
    \centering
    \includegraphics[width=\linewidth]{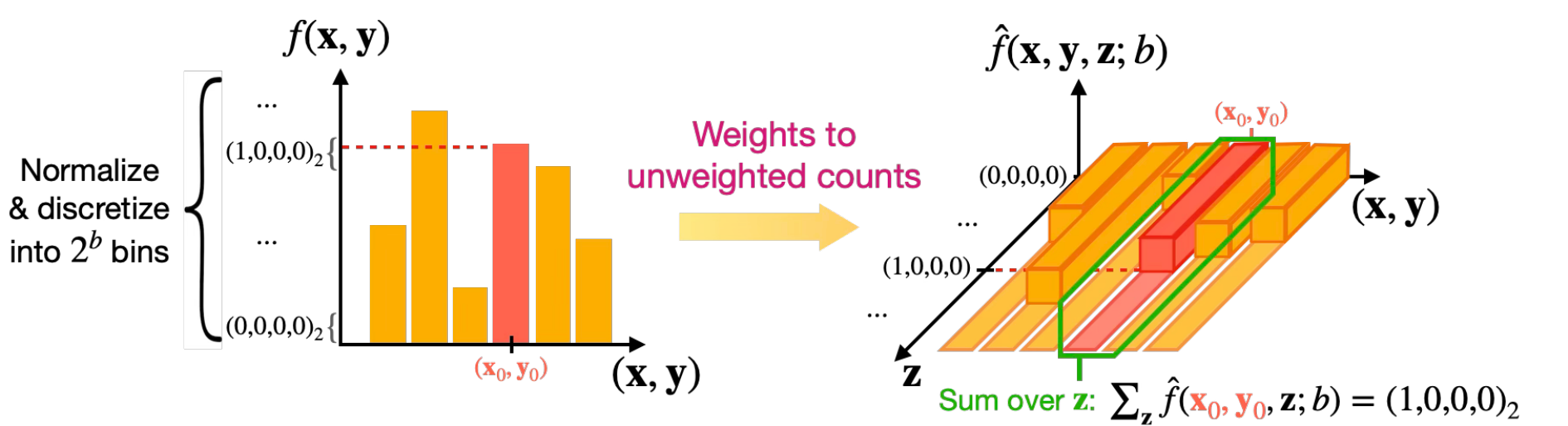}
    \caption{
    {Converting a weighted function $f(\mathbf{x}, \mathbf{y})$ into unweighted model counting $\sum_{\mathbf{z}} \hat{f}(\mathbf{x}, \mathbf{y}, \mathbf{z}; b)$.
    \textbf{(Left)} The function $f$ is normalized and uniformly discretized into $2^b$ integer levels, where $b \in \mathbb{Z}_{>0}$ is a user-specified discretization precision. The red bar shows an example $f(\mathbf{x}_0, \mathbf{y}_0)$, which is discretized to the binary integer $(1,0,0,0)_2$ (or $8$ in decimal).
    \textbf{(Right)} We construct an unweighted function $\hat{f}$ such that, when $\mathbf{z} \in \{0,1\}^b$ is interpreted as a binary integer, $\hat{f}(\mathbf{x}, \mathbf{y}, \mathbf{z}; b) = 1$ if $\mathbf{z}$ is no greater than the discretized $f(\mathbf{x}, \mathbf{y})$, and $0$ otherwise. Then $f(\mathbf{x}, \mathbf{y})$ can be directly computed as $\sum_{\mathbf{z}} \hat{f}$, e.g., the red bar shows that $\sum_{\mathbf{z}} \hat{f}(\mathbf{x}_0, \mathbf{y}_0, \mathbf{z}; b) = (1,0,0,0)_2$ exactly.}
    }
    \label{fig:w_to_uw}
\end{figure}

We first show how any weighted objective function can be embedded into an unweighted model counting formulation by introducing auxiliary binary variables.

{
The central idea is to replace real-valued weights with \emph{multiplicities of Boolean
assignments}. Instead of associating each assignment $(\mathbf{x},\mathbf{y})$ with a
numeric weight $f(\mathbf{x},\mathbf{y}) \in \mathbb{R}_{\ge 0}$, we introduce auxiliary
binary variables $\mathbf{z}$ and construct an \emph{unweighted} function
$\hat{f}(\mathbf{x},\mathbf{y},\mathbf{z})$ such that the number of satisfying $\mathbf{z}$
of $\hat{f}$ for fixed $(\mathbf{x},\mathbf{y})$ is proportional to $f(\mathbf{x},\mathbf{y})$, i.e., $\sum_{\mathbf{z}}\hat{f}(\mathbf{x},\mathbf{y},\mathbf{z}) \propto f(\mathbf{x},\mathbf{y})$.

To construct the unweighted function, we first normalize and discretize the value
$f(\mathbf{x}, \mathbf{y})$ into the range $[0, 2^b]$ for some
$b \in \mathbb{Z}_{\ge 0}$, and denote the resulting discretized value by
$\lfloor r_b(\mathbf{x}, \mathbf{y}) \rfloor$ (Definition~\ref{def:embedding}, using floor rounding).
For example, Figure~\ref{fig:w_to_uw} (left) shows that $f(\mathbf{x}_0, \mathbf{y}_0)$ is normalized to $(1,0,0,0)_2$. A larger value of $b$ yields higher approximation accuracy. The key trick is that, for any integer $B \in [0,2^b]$, this identity holds:
\[
\sum_{\mathbf{z} \in \{0,1\}^b}
\mathbf{1}\Big(
\sum_{i=1}^b 2^{i-1} z_i < B
\Big)
= B,
\]
where $\mathbf{1}(\cdot)$ denotes the indicator function and
$\mathbf{z} = (z_1,\dots,z_b) \in \{0,1\}^b$ are auxiliary binary variables.
This can be interpreted as the fact that exactly $B$ numbers 0 ... B-1 (represented as binary numbers in the vector $\mathbf{z}$) are smaller than $B$.

Using this observation, we define an indicator function
$\hat{f}(\mathbf{x},\mathbf{y},\mathbf{z};b)$ that evaluates to $1$ if and only if
\[
\sum_{i=1}^b 2^{i-1} z_i < \lfloor r_b(\mathbf{x},\mathbf{y}) \rfloor.
\]
Consequently, for each fixed $(\mathbf{x},\mathbf{y})$, the number of satisfying
assignments of $\mathbf{z}$ equals $\lfloor r_b(\mathbf{x},\mathbf{y}) \rfloor$, as shown in  Figure~\ref{fig:w_to_uw} (right). 
We thereby converted the discretized weight into an unweighted model count. We can prove that the unweighted count can faithfully approximate the original count (Lemma~\ref{lem:count-ineqality}). 
}

\begin{definition}[Embedding]
\label{def:embedding}
Let $f:\{0,1\}^n\times\{0,1\}^m\to\mathbb{R}_{\ge 0}$ and assume known bounds
\[
L \triangleq \min_{\mathbf{x},\mathbf{y}} f(\mathbf{x}, \mathbf{y}), \qquad
U \triangleq \max_{\mathbf{x},\mathbf{y}} f(\mathbf{x},\mathbf{y}), \qquad U > L.
\]
Let $b\in\mathbb{N}$ be the number of additional binary variables (bits), denoted $\mathbf{z}=(z_1,\ldots,z_b)\in\{0,1\}^b$.  
Define the scaled value
\[
r_b(\mathbf{x},\mathbf{y}) \triangleq \frac{f(\mathbf{x},\mathbf{y})-L}{U-L}\cdot 2^{b} \in [0,2^{b}].
\]
For each $\mathbf{x}\in\{0,1\}^n$, define the embedding 
\[
\mathcal{S}_{\mathbf{x}}(f,b)
\triangleq
\Big\{ (\mathbf{y},\mathbf{z})\in \{0,1\}^m\times\{0,1\}^b  \Big|  \sum_{i=1}^{b} 2^{i-1} z_i < r_b(\mathbf{x},\mathbf{y}) \Big\}.
\]
\end{definition}

\begin{lemma}[Discretized Weighted Count]
\label{lem:count-ineqality}
Let $\hat{f}(\mathbf{x},\mathbf{y},\mathbf{z};b)$ be the indicator that equals $1$ iff $(\mathbf{y},\mathbf{z}) \in \mathcal{S}_{\mathbf{x}}(f,b)$. Then
\[
2^{m} L + \frac{U-L}{2^{b}} \sum_{\mathbf{y},\mathbf{z}} \hat{f}(\mathbf{x},\mathbf{y},\mathbf{z};b)
\le
\sum_{\mathbf{y}} f(\mathbf{x},\mathbf{y})
<
2^{m} L + \frac{U-L}{2^{b}} \sum_{\mathbf{y},\mathbf{z}} \hat{f}(\mathbf{x},\mathbf{y},\mathbf{z};b)
+ \frac{2^{m}(U-L)}{2^{b}}.
\]
\end{lemma}

\begin{proof}
Fix $\mathbf{x}\in\{0,1\}^n$. By the definition of $\mathcal{S}_{\mathbf{x}}(f,b)$, for each $\mathbf{y}$ the number of $\mathbf{z}\in\{0,1\}^b$
satisfying $\sum_{i=1}^b 2^{i-1}z_i < r_b(\mathbf{x},\mathbf{y})$ equals $\lfloor r_b(\mathbf{x},\mathbf{y})\rfloor$.
Equivalently,
\begin{equation}\label{eq:count-equals-floor}
\sum_{\mathbf{z}} \hat{f}(\mathbf{x},\mathbf{y},\mathbf{z};b)=\lfloor r_b(\mathbf{x},\mathbf{y})\rfloor.
\end{equation}
Since $\lfloor r_b(\mathbf{x},\mathbf{y})\rfloor\le r_b(\mathbf{x},\mathbf{y})< \lfloor r_b(\mathbf{x},\mathbf{y})\rfloor+1$, multiplying by $(U-L)/2^b$ yields the
pointwise bounds
\begin{equation}\label{eq:ptwise}
\frac{U-L}{2^b}\lfloor r_b(\mathbf{x},\mathbf{y})\rfloor
\le
f(\mathbf{x},\mathbf{y})-L
<
\frac{U-L}{2^b}\big(\lfloor r_b(\mathbf{x},\mathbf{y})\rfloor+1\big).
\end{equation}
Summing \eqref{eq:ptwise} over all $\mathbf{y}\in\{0,1\}^m$ and using
$\sum_{\mathbf{y}} 1 = 2^m$ and \eqref{eq:count-equals-floor} gives
\[
2^m L + \frac{U-L}{2^b}\sum_{\mathbf{y},\mathbf{z}} \hat{f}(\mathbf{x},\mathbf{y},\mathbf{z};b)
\le
\sum_{\mathbf{y}} f(x,y)
<
2^m L + \frac{U-L}{2^b}\sum_{\mathbf{y},\mathbf{z}} \hat{f}(\mathbf{x},\mathbf{y},\mathbf{z};b) + \frac{2^m(U-L)}{2^b}.
\]
\end{proof}

By introducing $b$ auxiliary variables, we can approximate the weighted count $\sum_{\mathbf{y}} f(\mathbf{x},\mathbf{y})$ through the unweighted model count 
\begin{align}
\label{eq:estimate-w-count}
    2^{m} L + \frac{U-L}{2^{b}} \sum_{\mathbf{y},\mathbf{z}} \hat{f}(\mathbf{x}, \mathbf{y},\mathbf{z};b).
\end{align}
Hence, the original weighted objectives can be replaced with equivalent unweighted formulations parameterized by~$b$, providing a controllable trade-off between accuracy and computational cost.

\subsection{Step 2: Constructing a Pseudo-SMOO Problem to Narrow the Approximation Gap}


{
The reduction in Step~1 converts weighted model counting objectives into unweighted ones.
A natural approach is therefore to directly replace weighted objectives with their
unweighted counterparts and solve the resulting problem using \method.
However, this naive replacement retains the irreducible approximation factor
$\epsilon \ge 3$ (required by Theorems~\ref{thm:pareto-curve} and~\ref{thm:pareto-solution}), which
limits the achievable accuracy.

To address this, notice that the approximation gap for estimating unweighted model counts is always a
constant factor; that is, the approximation factor is a user-specified constant $\epsilon$,
independent to the specific objective function
(Theorems~\ref{thm:pareto-curve} and~\ref{thm:pareto-solution}).
This allows us to reduce the effective approximation error using a simple
amplification trick.
If an estimator approximates $\sum_y f(y)$ within a constant factor $2^\epsilon$, then
estimating
\[
\sum_{y_1}\cdots\sum_{y_T} f(y_1) \dots f(y_T)
= \bigl(\sum_y f(y)\bigr)^T
\]
can achieve the same constant-factor approximation, which implies that the relative error
in estimating $\sum_y f(y)$ decreases to $2^{\epsilon/T}$.

In this step, we construct an unweighted \emph{pseudo-SMOO} problem using the amplification
idea. By carefully designing this \emph{pseudo-SMOO} formulation, we can reuse the
Algorithm~\ref{alg:main} to solve it, and the resulting solution arbitrarily closely approximates the true Pareto frontier.
}

\begin{definition}[Pseudo-SMOO Problem]
\label{def:pseudo-smoop}
For the SMOO problem defined in Equation~\eqref{eq:smoop}, we build a new problem as follows.  
Define the indicator function from the previous step
\[
\hat{f}_i(\mathbf{x},\mathbf{y}_i,\mathbf{z}_i;b) \triangleq \mathbf{1}\big[(\mathbf{y}_i, \mathbf{z}_i) \in \mathcal{S}_{\mathbf{x}}(f_i, b)\big],
\]
and further define
\[
\hat{f}_i^{[T]}(\mathbf{x}, \{\mathbf{y}_i^{(t)}\}_{t=1}^T, \{\mathbf{z}_i^{(t)}\}_{t=1}^T; b)
\triangleq
\prod_{t=1}^T \hat{f}_i^{(t)}(\mathbf{x}, \mathbf{y}_i^{(t)}, \mathbf{z}_i^{(t)}; b),
\]
where $\hat{f}_i^{[T]}:\{0,1\}^n \times \{0,1\}^{T(|\mathbf{y}_i|+b)} \to \{0,1\}$ represents $T$ independent repetitions of $\hat{f}_i$, 
each corresponding to an independent copy $(\mathbf{y}_i^{(t)}, \mathbf{z}_i^{(t)})$ of the latent variables. 
The resulting pseudo-SMOO problem is
\begin{align}
\label{eq:convert-smoop}
\max_{\mathbf{x}} 
\Big( 
\sum_{\{\mathbf{y}_1^{(t)},\mathbf{z}_1^{(t)}\}_{t=1}^T}  \hat{f}_1^{[T]}, 
\dots, 
\sum_{\{\mathbf{y}_k^{(t)},\mathbf{z}_k^{(t)}\}_{t=1}^T}  \hat{f}_k^{[T]}
\Big).
\end{align}
\end{definition}

This pseudo problem can be solved using the same algorithmic framework (Algorithm~\ref{alg:main}) as the unweighted case, and its solutions can be shown to closely approximate those of the original weighted problem.

\subsection{Step 3: Bridging Pseudo-SMOO and Original Problem Solutions}

Finally, we establish that solutions obtained from the pseudo-SMOO problem remain valid approximations for the original problem, with a quantifiable approximation bound.

\begin{algorithm}[!h]
\caption{$\texttt{w-\method}(\sum f_1,\dots, \sum f_k, \delta, T, b)$}
\label{alg:weighted-main}
\begin{algorithmic}[1]
\Require Objective functions $\{\sum f_i \}_{i=1}^k$; error probability bound $\delta$; amplification factor $T$; discretization bit budget $b$.

\State $\hat{f}_i(\mathbf{x},\mathbf{y}_i,\mathbf{z}_i;b) \triangleq \mathbf{1}\big[(\mathbf{y}_i, \mathbf{z}_i) \in \mathcal{S}_{\mathbf{x}}(f_i, b)\big], \quad i=1,\dots,k.$
\State $\hat{f}_i^{[T]}(\mathbf{x}, \{\mathbf{y}_i^{(t)}\}_{t=1}^T, \{\mathbf{z}_i^{(t)}\}_{t=1}^T; b)
\triangleq
\prod_{t=1}^T \hat{f}_i^{(t)}(\mathbf{x}, \mathbf{y}_i^{(t)}, \mathbf{z}_i^{(t)}; b), \quad i=1,\dots,k.$

\State $\mathcal{F}_{dom} \gets \texttt{\method}(\sum \hat{f}_1^{[T]}, \dots, \sum \hat{f}_k^{[T]}, \delta, \epsilon=3)$

\State \Return $\mathcal{F}_{dom}$.
\end{algorithmic}
\end{algorithm}

The algorithm \texttt{w-\method} returns, with high probability, a set of solutions that forms an approximate Pareto frontier, as established by Theorem~\ref{thm:w-pareto-solution}, together with Corollary~\ref{cly:w-pareto-solution}. The corresponding proofs are given below.

\begin{proof}[]
\noindent\paragraph{Proof of Theorem~\ref{thm:w-pareto-solution}}
Similar to the proof for the unweighted case, we condition on the event~\eqref{eq:eventE}, which occurs with probability at least $1-\delta$.
On this event, the guarantees in Definition~\ref{def:sat-oracle} hold deterministically.

\medskip
\noindent \textbf{1. Relating weighted and unweighted objectives.}

Fix a Pareto-optimal solution $\mathbf{x}_{\mathrm{opt}}$ achieving
\[
F_i(\mathbf{x}_{\mathrm{opt}}) = \sum_{\mathbf{y}_i} f_i(\mathbf{x}_{\mathrm{opt}},\mathbf{y}_i), \qquad i=1,\dots,k.
\]
Let its corresponding unweighted model count in the converted problem in Definition~\ref{def:pseudo-smoop} be
\[
\widehat{F}^{[T]}_{i}(\mathbf{x}_{\mathrm{opt}};b) = \sum_{ \{\mathbf{y}^{(t)}_i, \mathbf{z}^{(t)}_i \}_{t=1}^T} \hat{f}^{[T]}_{i} (\mathbf{x}_{\mathrm{opt}}, \{\mathbf{y}_i\}_{t=1}^T, \{\mathbf{z}_i\}_{t=1}^T;b), \quad i=1,2,\dots,k.
\]
and define the decomposed form
\[
\widehat{F}_{i}(\mathbf{x}_{\mathrm{opt}};b) = \sum_{ \mathbf{y}_i, \mathbf{z}_i} \hat{f}_{i} (\mathbf{x}_{\mathrm{opt}}, \mathbf{y}_i, \mathbf{z}_i;b) = \Big(\widehat{F}^{[T]}_{i}(\mathbf{x}_{\mathrm{opt}};b) \Big)^{\frac{1}{T}}, \qquad i=1,2,\dots,k.
\]
From Lemma~\ref{lem:count-ineqality},
\begin{align}
F_i(\mathbf{x}_{\mathrm{opt}})
< 2^{|\mathbf{y}_i|}L_i
+ \frac{U_i-L_i}{2^b}\widehat{F}_i(\mathbf{x}_{\mathrm{opt}};b)
+ \frac{2^{|\mathbf{y}_i|}(U_i-L_i)}{2^b},    
\label{eq:approx-ineq1}
\end{align}
where
\[
L_i \triangleq \min_{\mathbf{x},\mathbf{y}_i} f_i(\mathbf{x},\mathbf{y}_i), \quad
U_i \triangleq \max_{\mathbf{x},\mathbf{y}_i} f_i(\mathbf{x},\mathbf{y}_i), \quad U_i>L_i.
\]

\medskip
\noindent \textbf{2. Applying the SAT oracle guarantees.}

When running \method with input $(\sum \hat{f}_1^{[T]}, \dots, \sum \hat{f}_k^{[T]}, \delta, \epsilon)$ where $\epsilon$ fixed at $3$, let
\[
l_i = \big\lfloor \log_2 \widehat{F}^{[T]}_i(\mathbf{x}_{\mathrm{opt}};b) \big\rfloor - l^*, \quad l^* = 2.
\]
Then $2^{l_i+l^*} \le \widehat{F}^{[T]}_i(\mathbf{x}_{\mathrm{opt}};b) < 2^{l_i+l^*+1}$ and $l^* \geq 2$,
ensuring the \emph{Guaranteed SAT} condition in Definition~\ref{def:sat-oracle}.  
By Theorem~\ref{thm:pareto-solution}, $\texttt{\method}(\sum \hat{f}_1^{[T]},\dots, \hat{f}_k^{[T]},\delta, 3)$ returns 
$(\mathbf{x}^\star,\mathbf{p}^\star)$ such that
\[
\widehat{F}^{[T]}_i(\mathbf{x}^\star;b) \ge 2^{l_i-l^*}, \quad \forall i.
\]
Consequently,
\[
\widehat{F}_i(\mathbf{x}^\star;b)
= \big(\widehat{F}^{[T]}_i(\mathbf{x}^\star;b)\big)^{1/T}
\ge 2^{\frac{l_i-l^*}{T}}, \quad \forall i.
\]

\medskip
\noindent \textbf{3. Relating actual objective values.}

From Lemma~\ref{lem:count-ineqality}, for $\mathbf{x}^\star$ we have
\[
F_i(\mathbf{x}^\star)
\ge 2^{|\mathbf{y}_i|}L_i + \frac{U_i-L_i}{2^b}\widehat{F}_i(\mathbf{x}^\star;b), \quad \forall i.
\]
Combining the above inequality with Equation~\eqref{eq:approx-ineq1} yields
\begin{align*}
   F_i(\mathbf{x}_{\mathrm{opt}}) &< 2^{|\mathbf{y}_i|} L_i + \frac{U_i-L_i}{2^b} \widehat{F}_i(\mathbf{x}_{\mathrm{opt}};b) + \frac{2^{|\mathbf{y}_i|}(U_i-L_i)}{2^b} \\
   & = 2^{|\mathbf{y}_i|} L_i + \frac{U_i-L_i}{2^b} \big( \widehat{F}^{[T]}_i(\mathbf{x}_{\mathrm{opt}};b) \big)^{\frac{1}{T}} + \frac{2^{|\mathbf{y}_i|}(U_i-L_i)}{2^b} \\
   & < 2^{|\mathbf{y}_i|} L_i + \frac{U_i-L_i}{2^b} 2^{\frac{l_i+l^*+1}{T}} + \frac{2^{|\mathbf{y}_i|}(U_i-L_i)}{2^b} \\
   & \leq 2^{|\mathbf{y}_i|} L_i + \frac{U_i-L_i}{2^b} 2^{\frac{(l_i-l^*) + 2l^*+1}{T}} + \frac{2^{|\mathbf{y}_i|}(U_i-L_i)}{2^b} \\
   & \leq 2^{|\mathbf{y}_i|} L_i + \frac{U_i-L_i}{2^b} 2^{\frac{2l^*+1}{T}} \widehat{F}_i(\mathbf{x}^\star; b) + \frac{2^{|\mathbf{y}_i|}(U_i-L_i)}{2^b} \\
   & \leq 2^{ \frac{2l^*+1}{T}} \Big( 2^{|\mathbf{y}_i|} L_i + \frac{U_i-L_i}{2^b} \widehat{F}_i(\mathbf{x}^\star;b) \Big) + \frac{2^{|\mathbf{y}_i|}(U_i-L_i)}{2^b} \\
   & \leq 2^{ \frac{2l^*+1}{T}} F_i(\mathbf{x}^\star) + \frac{2^{|\mathbf{y}_i|}(U_i-L_i)}{2^b} 
\end{align*}

\medskip
\noindent \textbf{4. Converting additive to multiplicative slack.}

To convert the slack between $F_i(\mathbf{x}^\star)$ and $F_i(\mathbf{x}_{\mathrm{opt}})$ into a single multiplicative slack, choose $\zeta(b)\ge0$ satisfying
\[
2^{\frac{2l^*+1}{T}}(2^{\zeta(b)}-1) F_i(\mathbf{x}^\star)
\ge \frac{2^{|\mathbf{y}_i|}(U_i-L_i)}{2^b}.
\]
Using the trivial lower bound $F_i(\mathbf{x}^\star)=\sum_{\mathbf{y}_i} f_i(\mathbf{x},\mathbf{y}_i)\ge 2^{|\mathbf{y}_i|}L_i$, a sufficient condition is
\[
2^{\frac{2l^*+1}{T}}(2^{\zeta(b)}-1) 2^{|\mathbf{y}_i|}L_i \ge \frac{2^{|\mathbf{y}_i|}(U_i-L_i)}{2^{b}}
\Leftrightarrow
2^{\zeta(b)}-1 \ge \frac{U_i-L_i}{2^{\frac{2l^*+1}{T}} L_i 2^{b}}.
\]
whose minimal solution is
\[
\zeta(b)
= \max_i \log_2 \Big(1+\frac{U_i-L_i}{2^{\frac{2l^*+1}{T}} L_i 2^b}\Big).
\]
Hence,
\[
F_i(\mathbf{x}_{\mathrm{opt}})
\le 2^{\frac{2l^*+1}{T}+\zeta(b)} F_i(\mathbf{x}^\star),
\quad \forall i.
\]
Therefore, the set of returned solutions forms a 
$(2^{\frac{5}{T}+\zeta(b)})$-approximate Pareto frontier.
\end{proof}


\begin{proof}[]
\noindent\paragraph{Proof of Corollary~\ref{cly:w-pareto-solution}}
By Theorem~\ref{thm:w-pareto-solution}, with probability at least $1-\delta$, the returned solutions form a
$2^{\frac{5}{T}+\zeta(b)}$-approximate Pareto frontier.  
To obtain a target factor $\gamma$, it suffices to enforce
\[
2^{\frac{5}{T}+\zeta(b)} \le \gamma.
\]
A convenient way is to split the budget evenly between the two terms:
\[
2^{\frac{5}{T}} \le \gamma^{1/2}
\quad\text{and}\quad
2^{\zeta(b)} \le \gamma^{1/2}.
\]
The first inequality is satisfied by choosing $T = \frac{10}{\log_2(\gamma)}$.

For the second inequality, note that $2^{\zeta(b)}=\max_i\Big(1+\frac{U_i-L_i}{2^{\frac{5}{T}}L_i2^b}\Big)$.
Thus it is enough to require, for every $i$,
\[
2^b \ge \frac{U_i-L_i}{2^{\frac{5}{T}}L_i(\gamma^{1/2}-1)}.
\]
Using $2^{\frac{5}{T}}=\gamma^{1/2}$ under our choice of $T$, this becomes
\[
2^b \ge \frac{U_i-L_i}{L_i(\gamma-\gamma^{1/2})}
= \frac{U_i / L_i-1}{\gamma-\gamma^{1/2}}.
\]
Taking $\log_2$ and maximizing over $i$ yields
\[
b \ge \max_i \log_2(U_i / L_i-1) - \log_2(\gamma-\gamma^{1/2}),
\]
which matches the choice in the corollary. Substituting these parameters into Theorem~\ref{thm:w-pareto-solution}
gives $2^{\frac{5}{T}+\zeta(b)} \le \gamma$, completing the proof.
\end{proof}

\subsection{{Sizes of SAT Queries Solved by w-\method}}

Consider an SMOO problem defined in Equation~\eqref{eq:smoop}, where the $k$ objective functions are weighted counts of the form $\sum_{\mathbf{y}_i} f_i(\mathbf{x}, \mathbf{y}_i)$ for $i = 1, \dots, k$, the $n$ decision variables are $\mathbf{x} \in \{0,1\}^n$, and, for each $i$, $\mathbf{y}_i \in \{0,1\}^{|\mathbf{y}_i|}$ denotes the set of latent variables.

w-\method (Algorithm~\ref{alg:weighted-main}) solves this SMOO problem by encoding it into SAT queries. Each query asks whether the following formula is satisfiable:
\begin{align}
    \bigwedge_{i=1}^k \Psi_i(\mathbf{x}, \mathbf{y}_i^{(1,1)}, \dots, \mathbf{y}_i^{(m,T)}, \mathbf{z}_i^{(1,1)}, \dots, \mathbf{z}_i^{(m,T)})
    \label{eq:w-smoop-sat}
\end{align}
where
\begin{align*}
\Psi_i(\mathbf{x}, \mathbf{y}_i^{(1,1)}, \dots, \mathbf{y}_i^{(m,T)}, \mathbf{z}_i^{(1,1)}, \dots, \mathbf{z}_i^{(m,T)})
    = \mathtt{Majority}(\psi_i^{(1)}, \dots, \psi_i^{(m)}).
\end{align*}

\noindent\textbf{Variables.}
Each query determines satisfiability over the variables $\mathbf{x} \in \{0,1\}^n$, $\mathbf{y}_i^{(j,t)} \in \{0,1\}^{|\mathbf{y}_i|}$, and $\mathbf{z}_i^{(j,t)} \in \{0,1\}^b$ for $i = 1, \dots, k$, $j = 1, \dots, m$, and $t = 1, \dots, T$.
In total, each query involves
\[
n + mT \sum_{i=1}^k |\mathbf{y}_i| + mTkb
\]
Boolean variables.

\noindent\textbf{Constraints.}
Each SAT query is encoded using a MIP-style formulation involving Boolean variables and linear constraints.
It consists of the following formulas and constraints:
\begin{itemize}[leftmargin=15pt]
    \item \emph{Formulas Converting Weighted Functions to Unweighted Functions.}
    By Definition~\ref{def:pseudo-smoop}, with a user-specified amplification factor $T \in \mathbb{Z}_{> 0}$ and discretization bit budget $b \in \mathbb{Z}_{\ge 0}$, we can construct $\hat{f}_i^{[T]}(\mathbf{x}, \{\mathbf{y}_i^{(t)}\}_{t=1}^T, \{\mathbf{z}_i^{(t)}\}_{t=1}^T)$ from objective functions.
    
    \item \emph{Formulas Encoding XOR Counting.}
    There are $k \times m$ formulas of the form
    \begin{align*}
        \psi_i^{(j)}(\mathbf{x}, \{\mathbf{y}_i^{(j,t)}\}_{t=1}^T, \{\mathbf{z}_i^{(j,t)}\}_{t=1}^T) = & \hat{f}_i^{[T]}(\mathbf{x}, \{\mathbf{y}_i^{(j,t)}\}_{t=1}^T, \{\mathbf{z}_i^{(j,t)}\}_{t=1}^T) \land \\
        & \mathtt{XOR}_1(\{\mathbf{y}_i^{(j,t)}\}_{t=1}^T, \{\mathbf{z}_i^{(j,t)}\}_{t=1}^T)
        \land \dots \land \\
        & \mathtt{XOR}_\ell(\{\mathbf{y}_i^{(j,t)}\}_{t=1}^T, \{\mathbf{z}_i^{(j,t)}\}_{t=1}^T),
    \end{align*}
    \[
    i = 1, \dots, k,\quad j = 1, \dots, m,
    \]
    where $k$ is the number of objectives, $\delta \in (0,1)$ is the user-specified error probability bound, and
    \[
    m = \left\lceil 15\tau \right\rceil,
    \]
    \[
    \tau = \max\{\ln k,\, n \ln 2\}
    + \left(\sum_{i=1}^k \ln \left(|\mathbf{y}_i| + b\right) + k\ln T - \ln(\delta) + \ln 2\right).
    \]
    Here, $\mathtt{XOR}(\cdot)$ denotes an independent random XOR constraint.
    Each formula $\psi_i^{(j)}$ contains a number of XOR constraints ($\ell$) ranging from $0$ to $T(|\mathbf{y}_i| + b)$.
    
    \item \emph{Constraints Encoding the SAT Query~\eqref{eq:uw-smoop-sat}.}
    We introduce auxiliary Boolean variables $b_i^{(j)}$ with constraints
    \[
    b_i^{(j)} \Leftrightarrow \psi_i^{(j)}(\mathbf{x}, \{\mathbf{y}_i^{(j,t)}\}_{t=1}^T, \{\mathbf{z}_i^{(j,t)}\}_{t=1}^T),
    \quad i = 1, \dots, k,\; j = 1, \dots, m,
    \]
    and enforce majority constraints
    \[
    \sum_{j=1}^m b_i^{(j)} > \frac{m}{2}, \quad i = 1, \dots, k.
    \]
\end{itemize}

{
For the weighted version, instead of considering the weighted objective function $f_i$ directly, we measure the SAT query size relative to the encoding of the indicator function $\hat{f}_i$ (Lemma~\ref{lem:count-ineqality}), which indicates whether $f_i$ exceeds a given threshold.
In a SAT query, there are $O(m)$ constraints, and in each constraint, the encoding of the indicator function appears roughly $O(T)$ times when going from $\hat{f}_i$ to $\hat{f}_i^{[T]}$ (Definition~\ref{def:pseudo-smoop}). 
Therefore, the size of one SAT query can be measured as $O(mT)$.
Since the approximation factor $\gamma$ is a user-specified parameter independent of the problem size, it can be treated as a constant in the asymptotic analysis. Therefore, the problem size simplifies to
$O\left(
n
+k\log(|Y|+\log U)
+\log\tfrac{1}{\delta}
\right)$.
}

{ \textbf{Total number of SAT queries.}~~
The total number of SAT queries solved by w-\method is $T^k\prod_{i=1}^k (|\mathbf{y}_i|+b)$. To achieve a fixed approximation factor $\gamma > 1$ for a $\gamma$-approximate Pareto frontier, let $U = \max_{i \in [k]} U_i/L_i$ and $|Y|\triangleq \max_{i\in[k]}|\mathbf{y}_i|$, so that
$\prod_{i=1}^k(|\mathbf{y}_i|+b)\le(|Y|+b)^k$. According to Corollary~\ref{cly:w-pareto-solution}, we select the parameters: $b=\log_2(U-1)-\log_2(\gamma-\sqrt{\gamma})$ and $T=10/\log_2\gamma$.
Since constants and logarithm bases do not affect asymptotic order, we have
$T^k\in O((\log\gamma)^{-k})$ and
$b\in \Theta(\log(U/(\gamma-\sqrt{\gamma})))$.
As $\gamma\to1^+$, we have $\log\gamma\in \Theta(\gamma-1)$ and
$\gamma-\sqrt{\gamma}\in \Theta(\gamma-1)$, yielding the simplified bound
$O\big(((|Y|+\log U+\log\frac{1}{\gamma-1})/(\gamma-1))^k\big)$.}

\section{Experiment}

We evaluate the proposed method on two scenarios. The first scenario, \textit{Road Network Strengthening to Mitigate Seasonal Disruptions} (Sec.~\ref{sec:seasonal}), reflects a realistic and complex SMOO setting in which we search for optimal road network strengthening plans to optimize connectivity across two seasons, with traffic patterns varying stochastically. It is constructed from real-world road networks obtained from OpenStreetMap~\cite{OpenStreetMap} and incorporates seasonal disruption patterns derived from geographically grounded weather records from the Meteostat library~\cite{meteostat}.
The second scenario, \textit{Flexible Supply Chain Network Design} (Sec.~\ref{sec:supply}), where we aim at designing a robust supply chain network maximizing flexibility (e.g., the number of routes each material can be sourced) while minimizing cost. It is derived from standard TSPLIB~\cite{reinelt1991tsplib} benchmark instances, which are widely used in combinatorial optimization research.


Experimental results show that \method consistently found better Pareto frontiers compared to baselines. This can be justified from several aspects: 
\begin{enumerate}
\item\textbf{\textit{Better Pareto solutions}}: intuitively, this means that our methods find solutions that have the best objective values among those found by all solvers. This is reflected by the \textit{Generational Distance (GD)} metric.
\item\textbf{\textit{Better coverage}}: at a high level, this means that for every Pareto optimal solution, our method is more likely to find one closely approximating it. This is reflected by the \textit{Inverted Generational Distance (IGD)} and \textit{Hypervolume (HV)} metric.
\item\textbf{\textit{More evenly distributed solutions}}: this means that the solutions found by our method spread out more evenly in the entire domain, hence capturing a large portion of the Pareto frontier. This is reflected by the \textit{Spacing (SP)} metric.
\end{enumerate}
We can also verify the superiority of \method visually in Figure \ref{fig:seasonal_event} and \ref{fig:supply_pareto}. 
The performance advantage becomes more pronounced as the objective becomes more difficult, supporting the effectiveness of our proposed approach.

\subsection{Baselines}

For baseline methods, we include several widely used state-of-the-art multi-objective algorithms, namely AGE-MOEA~\cite{age-moea}, NSGA-II~\cite{nsga2}, RVEA~\cite{rvea}, C-TAEA~\cite{c-taea}, and SMS-EMOA~\cite{sms-emoa}, implemented in PyMOO~\cite{pymoo}. 
These algorithms represent diverse design principles, including dominance-based, hypervolume-based, reference-vector-based, and constraint-handling approaches, and have demonstrated strong empirical performance on standard benchmark problems.

%
%
Baseline solvers require exact evaluation of all objective functions, including model counting objectives.
Because our two applications involve different forms of counting (weighted and unweighted), we use different model counters to be embedded in these baseline optimizers. 
In Section~\ref{sec:seasonal}, the objective is weighted model counting for computing reachability probabilities under stochastic disruptions. 
For this, we use Toulbar2~\cite{cooper2010toulbar} to perform probabilistic inference. 
In Section~\ref{sec:supply}, the objective reduces to unweighted model counting. 
For this setting, we use GANAK-2.4.6~\cite{SM2025,SRSM19} to compute exact counts.

For simplicity of evaluation, we assume that all objective functions involving model counting can be computed within a short time frame ($\sim$10 minutes) under a fixed policy. 
This is not too short because each solver potentially needs to evaluate many (millions of) different policies.
Without this assumption, baseline methods often fail to produce any solution within a reasonable time frame (e.g., hours).
Although exact model counting could be highly intractable, our approach can still generate feasible approximate solutions.
Also, because the model counting under a fixed policy can be computed in a relatively short amount of time, we use the exact objective values when comparing the quality of the Pareto frontiers produced by different solvers.

\subsection{Metrics}

Since computing the true Pareto frontier is infeasible for the benchmark problems considered, we adopt a common strategy in the literature: constructing a \textit{reference Pareto frontier}, denoted by $\mathcal{P}$, by merging the non-dominated solutions obtained by all solvers under comparison. This aggregated frontier serves as a proxy for the ground truth and enables consistent evaluation across methods.

Let $\hat{\mathcal{P}}$ denote the approximate Pareto frontier returned by a solver. We evaluate its quality using the following metrics:
\begin{itemize}
\item \textbf{Generational Distance (GD)}: Measures the average distance from each solution in $\hat{\mathcal{P}}$ to the closest point in the reference frontier $\mathcal{P}$, reflecting the solution quality (aka. convergence):
\[
\mathrm{GD}(\hat{\mathcal{P}}) 
= \frac{1}{|\hat{\mathcal{P}}|} 
\sum_{x \in \hat{\mathcal{P}}} 
\min_{y \in \mathcal{P}} \| x - y \|.
\]

\item \textbf{Inverted Generational Distance (IGD)}: Measures the average distance from each point in the reference frontier $\mathcal{P}$ to the closest solution in $\hat{\mathcal{P}}$, reflecting coverage:
\[
\mathrm{IGD}(\hat{\mathcal{P}}) 
= \frac{1}{|\mathcal{P}|} 
\sum_{y \in \mathcal{P}} 
\min_{x \in \hat{\mathcal{P}}} \| y - x \|.
\]

\item \textbf{Hypervolume (HV)}: Measures the volume of the objective space dominated by $\hat{\mathcal{P}}$ and bounded by a reference point $r$:
\[
\mathrm{HV}(\hat{\mathcal{P}}) 
= \mathrm{vol} \Big( 
\bigcup_{x \in \hat{\mathcal{P}}} [x, r] 
\Big).
\]
It captures both convergence and diversity. A larger HV indicates a better approximation of the reference frontier.

\item \textbf{Spacing (SP)}: Measures the uniformity of distances between neighboring solutions in $\hat{\mathcal{P}}$. Let $d_i$ denote the minimum distance from solution $x_i \in \hat{\mathcal{P}}$ to any other solution in the same set, and let $\bar{d}$ be their mean. Then:
\[
\mathrm{SP}(\hat{\mathcal{P}}) 
= \sqrt{ 
\frac{1}{|\hat{\mathcal{P}}|-1} 
\sum_{i} (d_i - \bar{d})^2 
}.
\]
A smaller SP indicates a more evenly distributed set of solutions.

\end{itemize}

Since objective values vary across different problem instances, we normalize all objectives to the range $[0,1]$ to enable consistent and interpretable comparisons.

\subsection{Scenario 1: Road Network Strengthening to Mitigate Seasonal Disruptions}
\label{sec:seasonal}

Urban road networks are exposed to uncertain and potentially correlated disruptions such as severe weather and seasonal traffic variations. These disruptions may temporarily disable road segments and affect connectivity between critical locations.

We study the following planning problem: \textit{given uncertain seasonal disruptions, how should a limited number of road segments be strengthened to maximize the probability that a critical destination remains reachable?}
\begin{figure}[h]
    \centering
    \includegraphics[width=0.9\linewidth]{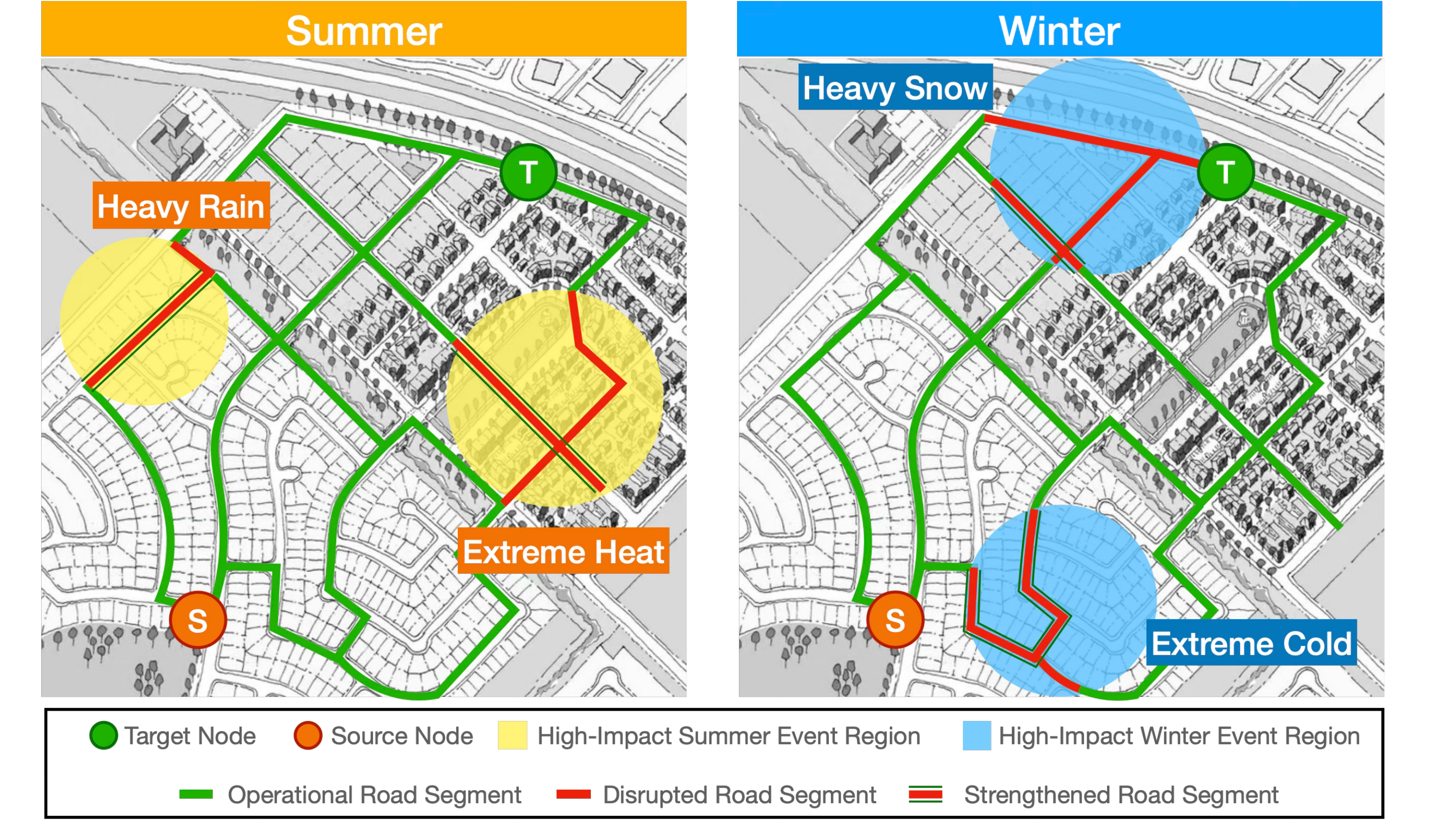}
    \caption{{Road Network Strengthening to Mitigate Seasonal Disruptions. Green edges denote operational road segments, while red edges represent disrupted segments under seasonal events. Yellow and blue regions indicate high-impact disruption areas in summer (e.g., heavy rain, extreme heat) and winter (e.g., heavy snow, extreme cold), respectively. The planner selects a limited set of road segments to strengthen in order to maintain connectivity between the source node $S$ and the target node $T$ under seasonal disruptions.}
}
    \label{fig:app_seasonal}
\end{figure}

\paragraph{Road Network and Decisions}
We model the transportation system as a road network $G = (V, E)$, where $V$ denotes the set of nodes (road intersections) and $E$ denotes the set of edges (road segments). 
Two special nodes are designated: a source node $s \in V$ (e.g., an emergency staging point) and a target node $t \in V$ (e.g., a hospital or evacuation site). We define reachability as whether the target remains reachable from the source within a fixed number of road segments (a hop limit $T$).

For each road segment, the planner makes a binary decision indicating whether the segment is strengthened. Let
\[
\mathbf{x} = (x_1, \dots, x_{|E|}) \in \{0,1\}^{|E|}
\]
denote the strengthening decisions, where $x_e = 1$ if road segment $e$ is strengthened, and $x_e = 0$ otherwise. Strengthened road segments remain operational even if affected by disruption events, whereas unstrengthened segments may become unavailable when certain events occur.

\paragraph{Seasonal Events and Road Disruptions}
Disruptions are modeled as binary random events. Each event represents a localized disturbance, such as heavy snowfall, high winds, extreme heat, or heavy rainfall.
Let
\[
\mathbf{s} = (s_1, \dots, s_{|S|}) \in \{0,1\}^{|S|}
\]
denote the indicators of seasonal disruption events, where $s_i = 1$ indicates that event $i$ occurs.

Each event is associated with a specific subset of road segments. If an event occurs, the associated roads become unavailable unless they have been strengthened (i.e., unless $x_e = 1$).

\paragraph{SMOOP Formulation}
We formulate a stochastic multi-objective optimization (SMOO) problem that maximizes seasonal connectivity probabilities:
\[
\max_{\mathbf{x}} 
\Big(
\sum_{\mathbf{s} \in S} \Pr_{\text{summer}}(\mathbf{s})  
\mathbb{I}\big[\text{$s,t$ connected} |\mathbf{s}, \mathbf{x}\big],
\sum_{\mathbf{s} \in S} \Pr_{\text{winter}}(\mathbf{s})  
\mathbb{I}\big[\text{$s,t$ connected} |\mathbf{s}, \mathbf{x}\big]
\Big)
\]
where $\mathbb{I}[\cdot]$ is an indicator function that equals $1$ if the source and target remain connected within hop limit $T$ under disruption events $\mathbf{s}$ and strengthening decisions $\mathbf{x}$, and equals $0$ otherwise.

\paragraph{Experimental Setting}

We construct a real road network using OpenStreetMap data centered at Central Park, New York, US, within a chosen radius for demonstration. The size of the graph varies with the selected radius (500m, 1000m, and 1500m). Connectivity is measured using maximum hop limits of 8, 10, and 12, respectively.

Because strengthening all roads is unrealistic, we impose a budget constraint such that at most $10\%$ of the total road segments may be strengthened: $\sum_{e \in E} x_e \leq 0.1 |E|.$

In our experiments, we used 12, 30, and 50 disruption events for networks of increasing sizes. The joint distribution of events is fitted to historical weather data near Central Park. Each optimization run is given a time limit of one hour.

\paragraph{Results}
Table~\ref{tab:seasonal_event} reports the Pareto frontier quality metrics under a one-hour time limit across different radii (500m, 1000m, and 1500m). 
Lower GD, IGD, and SP indicate better performance, while higher HV indicates better performance. 
Across all settings, \method achieves the lowest (or tied-lowest) GD, the lowest IGD, the highest HV, and competitive or lowest SP values, indicating stronger convergence, broader frontier coverage, and better solution distribution. Figure~\ref{fig:seasonal_event} visualizes the corresponding Pareto frontier curves. 
Since both objectives are maximized, solutions closer to the upper-right corner are more desirable. 
The curves produced by \method extend further toward the upper-right region and recover a broader Pareto frontier compared to baseline methods.

The main reason \method achieves better solution quality is that baseline solvers rely on iterative search. 
They must generate candidate solutions, evaluate them exactly, and gradually improve the population. 
This process is time-consuming, and under a fixed time limit, they may not explore enough of the objective space to find the best trade-offs. 
In some difficult cases, baseline solvers may even fail to produce meaningful Pareto solutions within the time limit. 
Although they may reach good solutions given unlimited time, they do not perform as well under strict time constraints.

In contrast, \method does not depend on iterative evolutionary search. 
Instead, it divides the objective space into grids and checks feasibility using satisfiability queries. 
This allows the method to systematically explore the entire objective space, including extreme (corner) trade-offs that evolutionary solvers may miss due to random initialization and stochastic updates. 
As a result, \method achieves better coverage of the Pareto frontier and produces more uniformly distributed solutions.

In summary, by reducing SMOOP to a satisfiability problem, \method can search more efficiently within a fixed time budget and obtain higher-quality, better-distributed Pareto solutions.

\begin{table*}[h]
\centering
\caption{Pareto frontier quality metrics for the Road Network Strengthening under Seasonal Disruptions scenario. 
Lower GD, IGD, and SP indicate better performance ($\downarrow$), while higher HV indicates better performance ($\uparrow$). All reported values are given as mean $\pm$ standard deviation over five independent runs, each conducted under a one-hour time limit.
Cell colors denote ranking: \textcolor{blue!80}{deep blue} indicates the best performance, and \textcolor{blue!40}{light blue} indicates the second best.
Across all radii (500m, 1000m, 1500m), \method consistently achieves the lowest (or tied-lowest) GD, indicating that its solutions are closest to the reference Pareto frontier and thus exhibit the best solution quality. It also attains the lowest IGD and highest HV scores, demonstrating superior coverage of the frontier, strong approximation across all regions of the Pareto set, and successful discovery of corner trade-off solutions. 
Finally, its low SP values indicate a more evenly-distributed solutions along the frontier, reflecting better diversity and stability.}

\resizebox{\textwidth}{!}{%
\begin{tabular}{c|l|cccc}
\toprule
\textbf{Radius} & \textbf{Solver} & GD$\downarrow$ & IGD$\downarrow$ & HV$\uparrow$ & SP$\downarrow$ \\
\midrule

\multirow{6}{*}{\begin{tabular}{c}
500m 
\end{tabular}}
& \method   &
$<10^{-6}$ &
\cellcolor{blue!30} \pmv{0.009}{0.004} &
\cellcolor{blue!30} \pmv{0.796}{0.007} &
\cellcolor{blue!30} \pmv{0.024}{0.005} \\

& NSGA2     &
$<10^{-6}$ &
\pmv{0.142}{0.037} &
\pmv{0.772}{0.014} &
\pmv{0.078}{0.053} \\

& AGE-MOEA  &
$<10^{-6}$ &
\pmv{0.099}{0.050} &
\cellcolor{blue!10} \pmv{0.787}{0.009} &
\pmv{0.155}{0.025} \\

& C-TAEA    &
$<10^{-6}$ &
\cellcolor{blue!10} \pmv{0.051}{0.004} &
\pmv{0.773}{0.008} &
\pmv{0.121}{0.031} \\

& RVEA      &
$<10^{-6}$ &
\pmv{0.177}{0.005} &
\pmv{0.733}{0.008} &
\cellcolor{blue!10} \pmv{0.027}{0.010} \\

& SMS-EMOA  &
$<10^{-6}$ &
\pmv{0.171}{0.022} &
\pmv{0.716}{0.031} &
\pmv{0.077}{0.016} \\

\midrule

\multirow{6}{*}{\begin{tabular}{c}
1000m 
\end{tabular}}
& \method   &
\cellcolor{blue!30} \pmv{0.002}{0.001} &
\cellcolor{blue!30} \pmv{0.045}{0.003} &
\cellcolor{blue!30} \pmv{0.861}{0.017} &
\cellcolor{blue!30} \pmv{0.051}{0.021} \\

& NSGA2     &
\pmv{0.006}{0.003} &
\pmv{0.176}{0.008} &
\pmv{0.851}{0.017} &
\pmv{0.059}{0.020} \\

& AGE-MOEA  &
\pmv{0.009}{0.002} &
\cellcolor{blue!10} \pmv{0.154}{0.016} &
\cellcolor{blue!10} \pmv{0.859}{0.003} &
\pmv{0.072}{0.028} \\

& C-TAEA    &
\pmv{0.008}{0.003} &
\pmv{0.200}{0.022} &
\pmv{0.808}{0.047} &
\pmv{0.082}{0.037} \\

& RVEA      &
\pmv{0.009}{0.005} &
\pmv{0.189}{0.019} &
\pmv{0.821}{0.012} &
\pmv{0.057}{0.008} \\

& SMS-EMOA  &
\cellcolor{blue!10} \pmv{0.002}{0.002} &
\pmv{0.211}{0.007} &
\pmv{0.800}{0.031} &
\cellcolor{blue!10} \pmv{0.051}{0.023} \\

\midrule

\multirow{6}{*}{\begin{tabular}{c}
1500m 
\end{tabular}}
& \method   &
\cellcolor{blue!30} \pmv{0.002}{0.001} &
\cellcolor{blue!30} \pmv{0.008}{0.002} &
\cellcolor{blue!30} \pmv{0.921}{0.008} &
\cellcolor{blue!30} \pmv{0.012}{0.005} \\

& NSGA2     &
\cellcolor{blue!10} \pmv{0.002}{0.003} &
\cellcolor{blue!10} \pmv{0.106}{0.009} &
\cellcolor{blue!10} \pmv{0.894}{0.001} &
\pmv{0.042}{0.014} \\

& AGE-MOEA  &
\pmv{0.017}{0.004} &
\pmv{0.257}{0.015} &
\pmv{0.876}{0.004} &
\pmv{0.040}{0.007} \\

& C-TAEA    &
\pmv{0.011}{0.002} &
\pmv{0.205}{0.016} &
\pmv{0.876}{0.005} &
\pmv{0.050}{0.004} \\

& RVEA      &
\pmv{0.020}{0.010} &
\pmv{0.242}{0.020} &
\pmv{0.877}{0.005} &
\cellcolor{blue!10} \pmv{0.028}{0.017} \\

& SMS-EMOA  &
\pmv{0.006}{0.001} &
\pmv{0.221}{0.024} &
\pmv{0.890}{0.012} &
\pmv{0.062}{0.022} \\

\bottomrule
\end{tabular}%
}
\label{tab:seasonal_event}
\end{table*}

\begin{figure*}[!h]
    \centering
    \begin{subfigure}[b]{0.45\linewidth}
        \includegraphics[width=\linewidth]{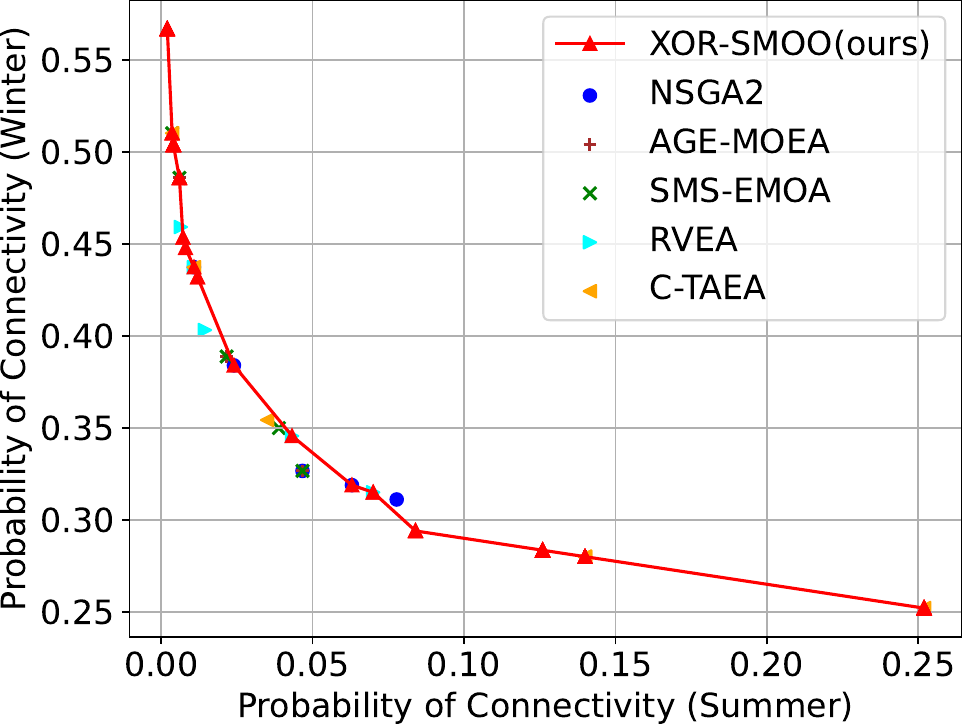}
        \caption{Radius 500m}
    \end{subfigure}
    \begin{subfigure}[b]{0.45\linewidth}
        \includegraphics[width=\linewidth]{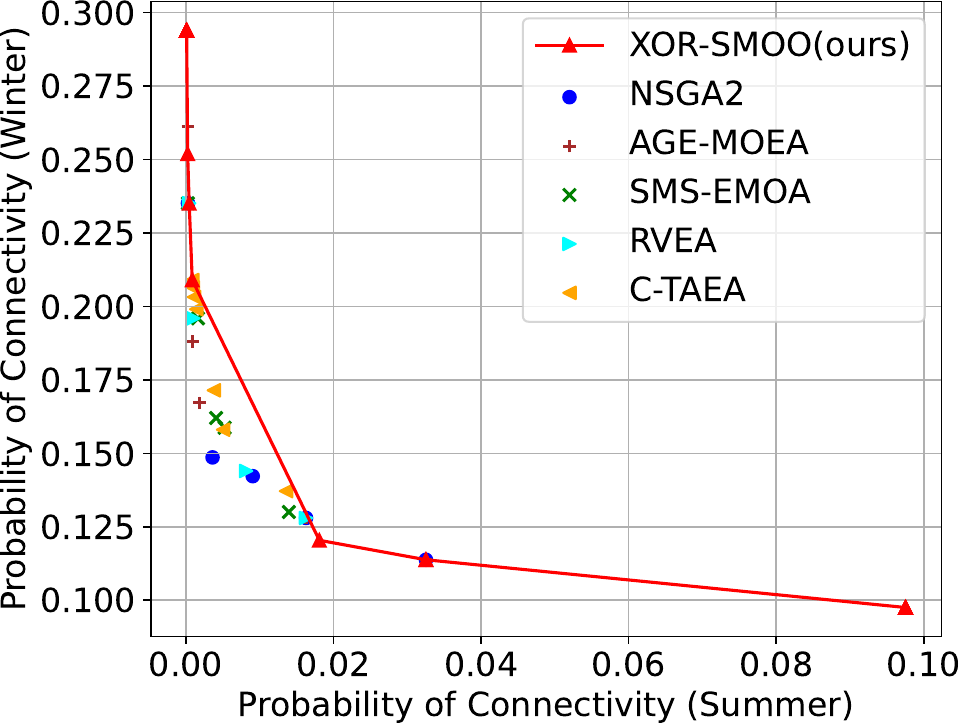}
        \caption{Radius 1000m}
    \end{subfigure}\\
    \begin{subfigure}[b]{0.45\linewidth}
        \includegraphics[width=\linewidth]{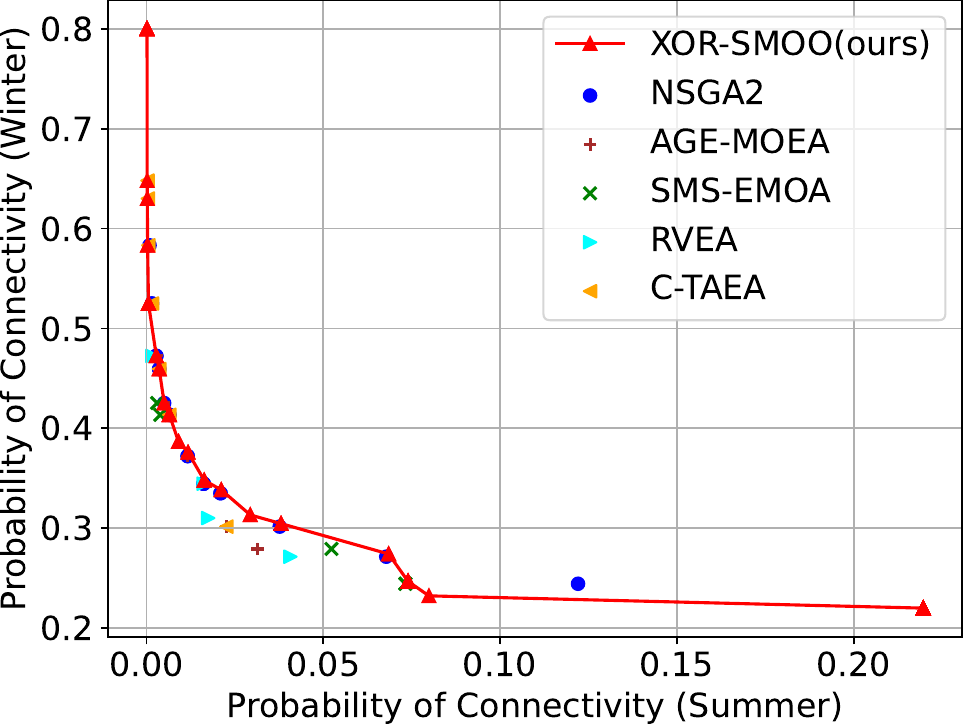}
        \caption{Radius 1500m}
    \end{subfigure}
    \caption{Pareto frontier curves for the Road Network Strengthening under Seasonal Disruptions scenario. Both objectives are maximized; therefore, solutions closer to the upper-right corner represent better trade-offs. Overall, \method consistently identifies higher-quality solutions and recovers a broader Pareto frontier, achieving better solution quality and Pareto frontier coverage compared to baseline methods.
    }
    \label{fig:seasonal_event}
\end{figure*}

\clearpage
\subsection{Scenario 2: Flexible Supply Chain Network Design}
\label{sec:supply}

A supply chain network connects supply sources to demand locations through intermediate transfer hubs. Activating more routes increases routing flexibility and robustness, but also increases construction or operational cost. A central planning question is:

\emph{Given a supplier and a demander, which subset of transportation routes should be activated to maximize delivery flexibility under a limited budget?}

\paragraph{Supply Chain Network and Decisions}

We model the supply chain network as a directed transportation network $G = (V, E)$, where $V$ denotes transfer hubs and $E$ denotes potential transportation routes between hubs. Two special nodes are designated: a \emph{supplier} node $s \in V$ and a \emph{demander} node $t \in V$.

\begin{wrapfigure}{r}{0.45\linewidth}
    \centering
    \vspace{-10pt}
    \includegraphics[width=0.9\linewidth]{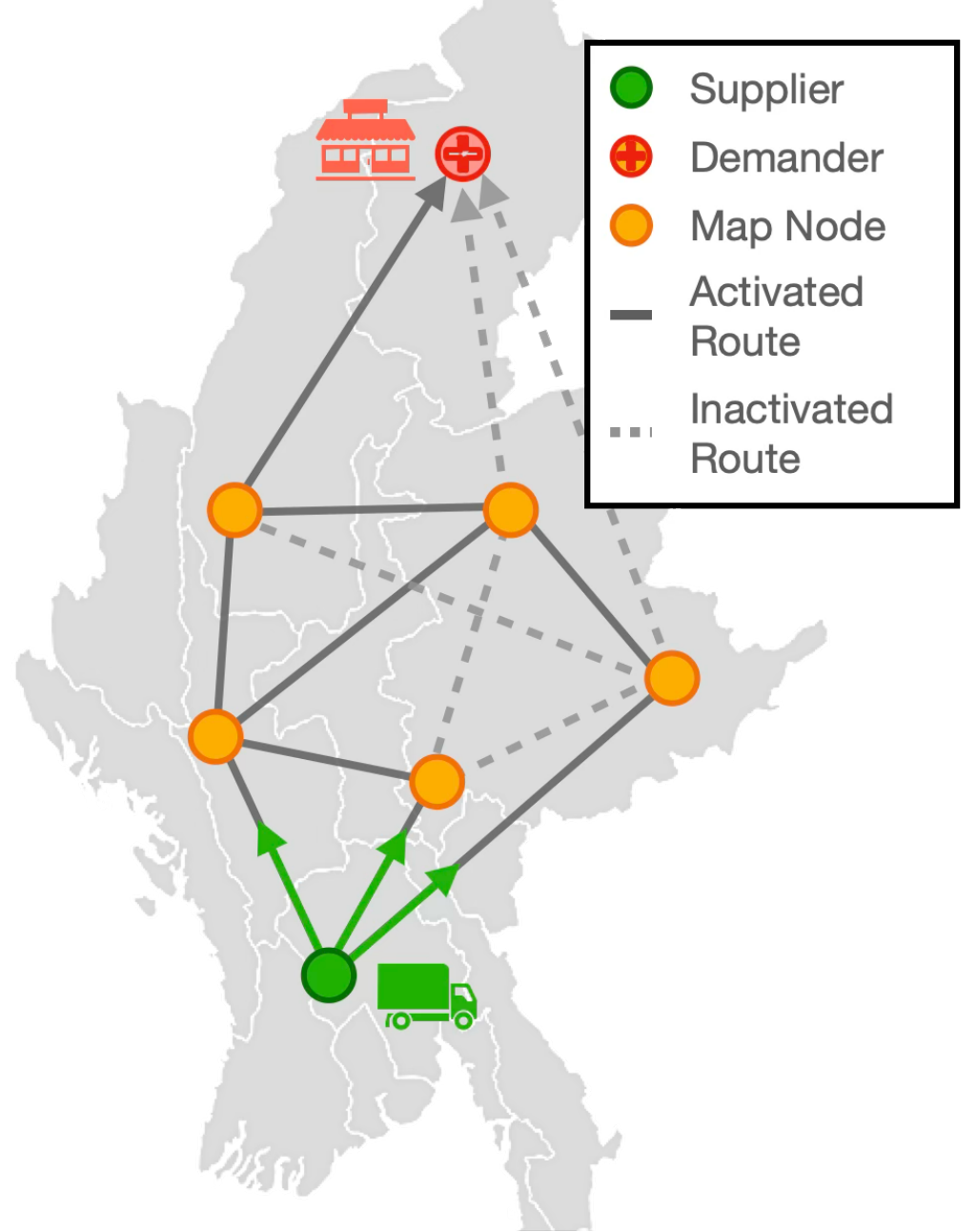}
    \caption{Flexible Supply Chain Network Design. The supplier (green) delivers goods to the demander (red) through intermediate transfer hubs (orange). Solid edges denote activated routes, while dashed edges are inactive. The goal is to balance delivery flexibility and cost through route selection.}
    \label{fig:supply}
    \vspace{-10pt}
\end{wrapfigure}

The planner selects which routes are activated. Let
\[
\mathbf{x} = (x_1, \dots, x_{|E|}) \in \{0,1\}^{|E|}
\]
denote the activation decisions, where $x_e = 1$ if route $e$ is active and $x_e = 0$ otherwise. Only activated routes may be used to transport goods.

\paragraph{Delivery Flexibility}

To characterize delivery flexibility, consider a binary vector
\[
\mathbf{f} = (f_1,\dots, f_{|E|}) \in \{0,1\}^{|E|}
\]
where $f_e = 1$ indicates that one unit of goods is sent through route $e$, and $f_e = 0$ otherwise.

A shipment pattern $\mathbf{f}$ is feasible if: (i) it only uses activated routes (i.e., $f_e \le x_e$ for all $e$), and  
(ii) goods are conserved at intermediate transfer hubs (what enters must leave),  
(iii) there exists a net shipment from supplier $s$ to demander $t$.

We quantify delivery flexibility by explicitly counting feasible shipment patterns:
\[
F(\mathbf{x})
=
\sum_{\mathbf{f} \in \{0,1\}^{|E|}}
\mathbb{I}\big[
\mathbf{f} \text{ forms a valid shipment under } \mathbf{x}
\big],
\]
where $\mathbb{I}[\cdot]$ equals $1$ if the shipment pattern is feasible and $0$ otherwise.

The value $F(\mathbf{x})$ reflects the structural flexibility of the activated network: a larger count indicates more distinct ways to route goods from supplier to demander. However, computing $F(\mathbf{x})$ requires checking feasibility over exponentially many binary configurations in $\{0,1\}^{|E|}$, which becomes intractable for large networks.

\paragraph{SMOOP Formulation}

To obtain a scale-independent objective, we normalize delivery flexibility by the full-network flexibility:
\[
F_{\max} := F(\mathbf{1}),
\qquad
\widetilde{F}(\mathbf{x}) = \frac{F(\mathbf{x})}{F_{\max}}.
\]
The quantity $\widetilde{F}(\mathbf{x}) \in [0,1]$ represents the fraction of total achievable delivery flexibility preserved under decision $\mathbf{x}$, and can be interpreted as \emph{flexibility retention}.

We model cost using the total length of activated routes. Let $d_e$ denote the distance of route $e$. The total cost is
$\sum_{e \in E} d_e \, x_e$, which captures construction or operational expenses. We therefore solve the following SMOO problem, balancing flexibility retention against total route distance:
\[
\max_{\mathbf{x}}
\Big(
\widetilde{F}(\mathbf{x}),
-
\sum_{e \in E} d_e \, x_e
\Big).
\]

\paragraph{Experimental Setting}
We evaluate the proposed method on transportation networks derived from TSPLIB benchmark instances. We use three maps: \texttt{Burma7}, \texttt{Burma14}, and \texttt{Ulysses16}, where the number indicates the number of cities in the instance. Each city is treated as a transfer hub, and edges correspond to transportation routes with Euclidean distances provided by TSPLIB. For each instance, the supplier $s$ and the demander $t$ are selected uniformly at random from the set of hubs. 

All experiments are repeated five independent times. Each run is given a time limit of one hour. Performance is evaluated using standard multi-objective metrics.

\paragraph{Results}

Table~\ref{tab:supply_chain} reports quantitative Pareto frontier quality metrics, and Figure~\ref{fig:supply_pareto} visualizes the corresponding Pareto curves.

Across all three maps, \method consistently achieves the lowest GD, indicating superior convergence toward the reference Pareto frontier. It also obtains the lowest IGD and the highest HV values in nearly all cases, demonstrating better coverage and approximation of the full Pareto frontier surface. In particular, \method more reliably identifies extreme trade-off solutions, which contributes to its strong HV performance. The SP values are also highest, indicating that its solutions are evenly distributed along the frontier.

The performance gap becomes more significant as the network size increases from Burma7 to Burma14 and Ulysses16. Since delivery flexibility is defined by counting feasible shipment patterns, the counting space grows exponentially with the number of routes. This significantly increases the difficulty of accurately evaluating and optimizing the flexibility objective. While baseline evolutionary methods degrade noticeably under this combinatorial growth, \method maintains strong convergence and coverage, leading to a widening performance gap on larger instances.

These results show that our XOR-counting-based method is particularly effective for problems with combinatorial model counting objectives.

\begin{table*}[!h]
\centering
\caption{Pareto frontier quality metrics for the Flexible Supply Chain Network Design problem. Lower GD, IGD, and SP indicate better performance ($\downarrow$), while higher HV indicates better performance ($\uparrow$). Values are mean $\pm$ standard deviation over five independent runs, with each run limited to one hour. Cell colors denote ranking: \textcolor{blue!80}{deep blue} indicates best and \textcolor{blue!40}{light blue} indicates second best. Results are reported for three TSPLIB instances (Burma7, Burma14, and Ulysses16), where the number denotes the number of transfer hubs.
Across all instances, \method achieves the lowest GD, indicating better solutions, and consistently superior IGD and HV, meaning better coverage of the Pareto frontier. Its competitive SP values further demonstrate more evenly distributed solutions. Noted that as the number of hubs increases, the flexibility objective becomes more challenging due to the exponential growth of feasible shipment patterns, and the performance gap between \method and baselines widens.
}

\resizebox{\textwidth}{!}{%
\begin{tabular}{c|l|cccc}
\toprule
\textbf{Radius} & \textbf{Solver} & GD$\downarrow$ & IGD$\downarrow$ & HV$\uparrow$ & SP$\downarrow$ \\
\midrule

\multirow{6}{*}{\begin{tabular}{c}
Burma7 
\end{tabular}}
& \method   & \cellcolor{blue!30} \pmv{0.0057}{0.0009} & \cellcolor{blue!30} \pmv{0.0174}{0.0054} & \cellcolor{blue!30} \pmv{0.2478}{0.0026} & \cellcolor{blue!30} \pmv{0.0423}{0.0058} \\

& NSGA2     & \pmv{0.0278}{0.0028} & \cellcolor{blue!10} \pmv{0.0497}{0.0092} & \cellcolor{blue!10} \pmv{0.2148}{0.0109} & \cellcolor{blue!10} \pmv{0.0462}{0.0032} \\

& AGE-MOEA  & \pmv{0.0214}{0.0033} & \pmv{0.0603}{0.0044} & \pmv{0.2091}{0.0022} & \pmv{0.0493}{0.0026} \\

& C-TAEA    & \cellcolor{blue!10} \pmv{0.0059}{0.0026} & \pmv{0.0884}{0.0129} & \pmv{0.1747}{0.0146} & \pmv{0.0573}{0.0059} \\

& RVEA      & \pmv{0.0172}{0.0024} & \pmv{0.0769}{0.0009} & \pmv{0.1881}{0.0021} & \pmv{0.0858}{0.0012} \\

& SMS-EMOA  & \pmv{0.0191}{0.0013} & \pmv{0.0667}{0.0032} & \pmv{0.1921}{0.0033} & \pmv{0.0556}{0.0162} \\

\midrule

\multirow{6}{*}{\begin{tabular}{c}
Burma14 
\end{tabular}}
& \method   & \cellcolor{blue!30} $<10^{-6}$ & \cellcolor{blue!30} \pmv{0.0014}{0.0014} & \cellcolor{blue!30} \pmv{0.1059}{0.0078} & \cellcolor{blue!30} \pmv{0.0188}{0.0074} \\

& NSGA2     & \pmv{0.0402}{0.0044} & \cellcolor{blue!10} \pmv{0.1550}{0.0104} & \pmv{0.0326}{0.0041} & \pmv{0.0203}{0.0046} \\

& AGE-MOEA  & \pmv{0.0239}{0.0038} & \pmv{0.1595}{0.0154} & \pmv{0.0278}{0.0027} & \pmv{0.0203}{0.0075} \\

& C-TAEA    & \pmv{0.0312}{0.0025} & \pmv{0.2213}{0.0081} & \cellcolor{blue!10} \pmv{0.0366}{0.0036} & \cellcolor{blue!10} \pmv{0.0192}{0.0047} \\

& RVEA      & \pmv{0.0988}{0.0047} & \pmv{0.2961}{0.0202} & \pmv{0.0253}{0.0021} & \pmv{0.0265}{0.0030} \\

& SMS-EMOA  & \cellcolor{blue!10} \pmv{0.0152}{0.0012} & \pmv{0.1579}{0.0158} & \pmv{0.0328}{0.0028} & \pmv{0.0303}{0.0107} \\

\midrule

\multirow{6}{*}{\begin{tabular}{c}
Ulysses16 
\end{tabular}}
& \method   & \cellcolor{blue!30}$<10^{-6}$ &\cellcolor{blue!30} $<10^{-6}$ & \cellcolor{blue!30} \pmv{0.0810}{0.0207} & \cellcolor{blue!30} \pmv{0.0134}{0.0213} \\

& NSGA2     & \pmv{0.0155}{0.0030} & \cellcolor{blue!10} \pmv{0.1688}{0.0123} & \pmv{0.0019}{0.0006} & \pmv{0.0183}{0.0052} \\

& AGE-MOEA  & \pmv{0.0186}{0.0030} & \pmv{0.1926}{0.0105} & \cellcolor{blue!10}\pmv{0.0036}{0.0002} & \cellcolor{blue!10} \pmv{0.0138}{0.0013} \\

& C-TAEA    & \cellcolor{blue!10} \pmv{0.0186}{0.0021} & \pmv{0.2782}{0.0103} & \pmv{0.0021}{0.0002} & \pmv{0.0153}{0.0012} \\

& RVEA      & \pmv{0.0421}{0.0021} & \pmv{0.3421}{0.0108} & \pmv{0.0035}{0.0003} & \pmv{0.0613}{0.0022} \\

& SMS-EMOA  & \pmv{0.0165}{0.0015} & \pmv{0.1801}{0.0079} & \pmv{0.0017}{0.0002} & \pmv{0.0188}{0.0030} \\

\bottomrule
\end{tabular}%
}
\label{tab:supply_chain}
\end{table*}

\begin{figure*}[h]
    \centering
    \begin{subfigure}[b]{0.45\linewidth}
        \includegraphics[width=\linewidth]{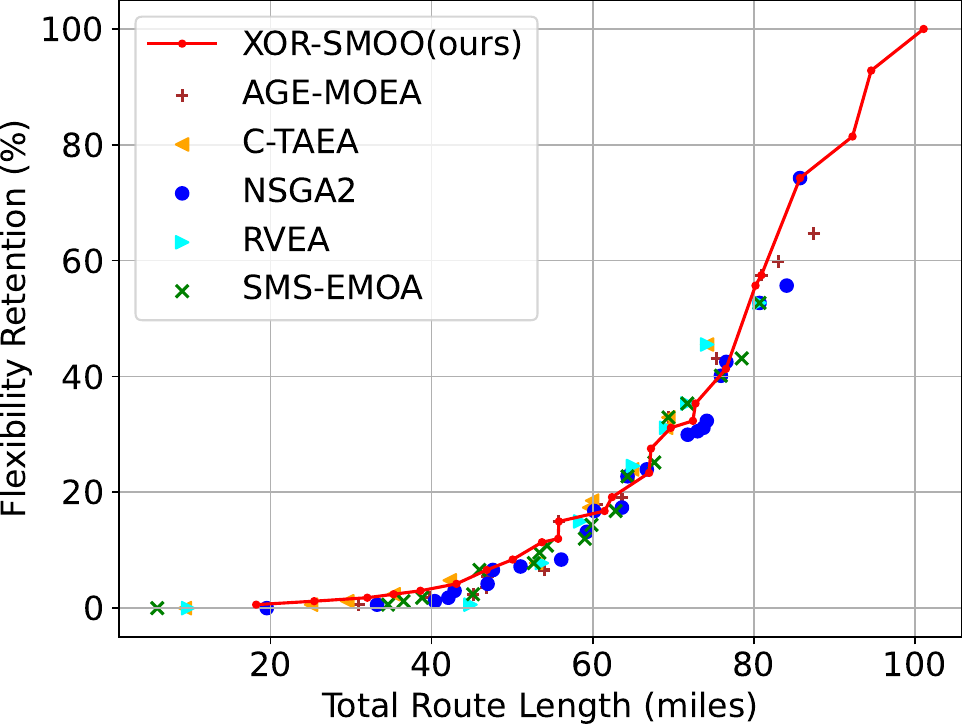}
        \caption{Burma 7 Cities}
    \end{subfigure}
    \begin{subfigure}[b]{0.45\linewidth}
        \includegraphics[width=\linewidth]{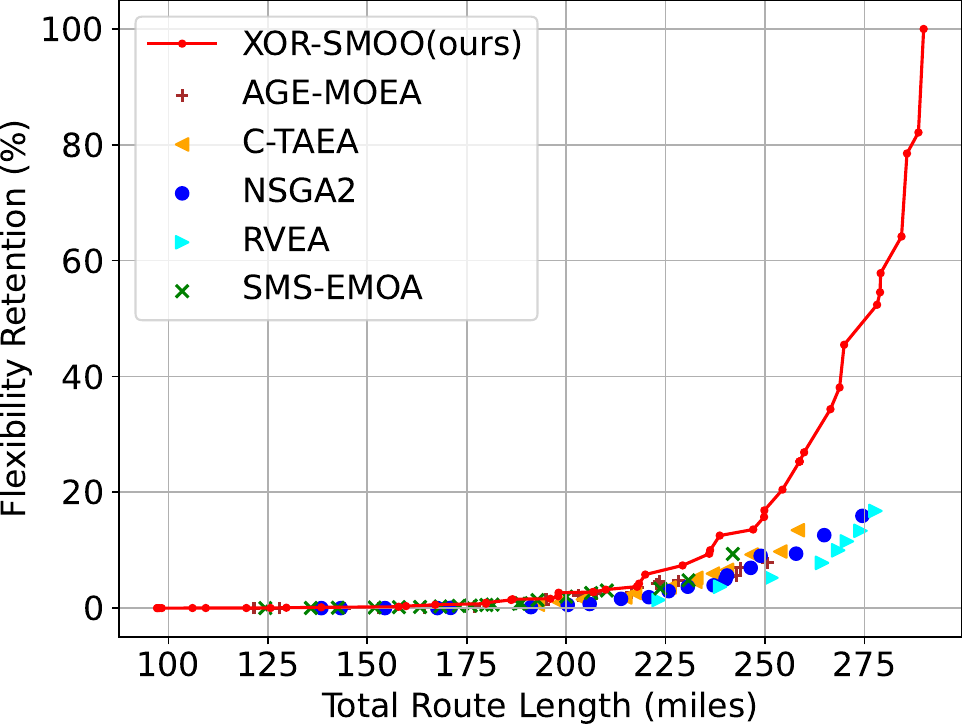}
        \caption{Burma 14 Cities}
    \end{subfigure}\\
    \begin{subfigure}[b]{0.45\linewidth}
        \includegraphics[width=\linewidth]{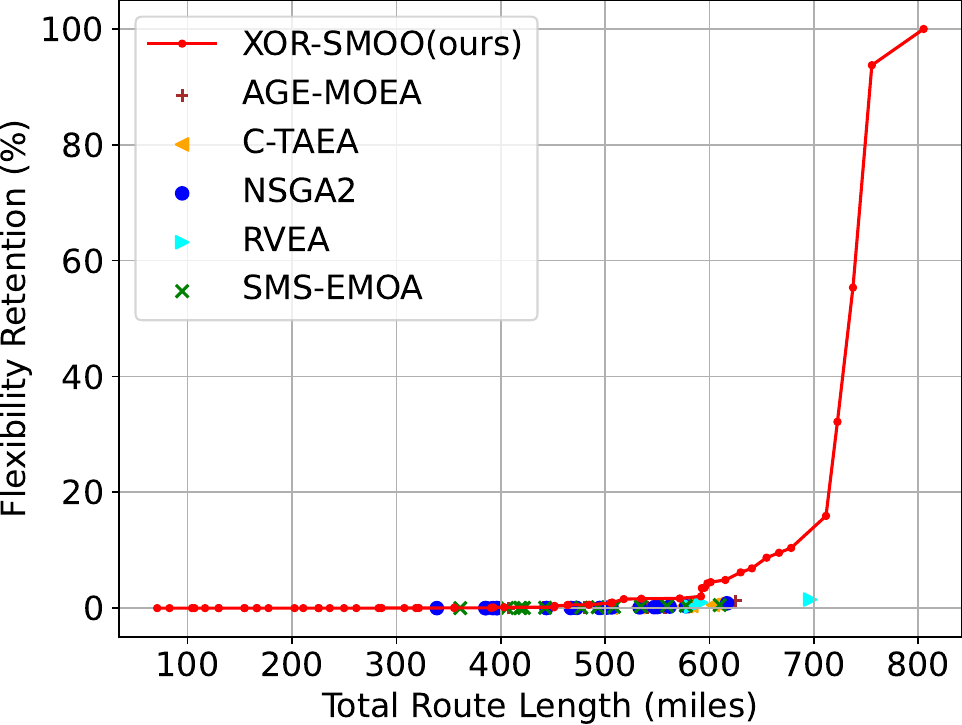}
        \caption{Ulysses 16 Cities}
    \end{subfigure}
    \caption{Pareto frontier curves for the Transportation Network Design problem. The flexibility objective is maximized while total length is minimized; therefore, solutions closer to the upper-left corner represent better trade-offs. Overall, \method consistently identifies higher-quality solutions and recovers a broader Pareto frontier, achieving superior convergence and coverage compared to baseline methods. As the number of transfer hubs increases, the combinatorial flexibility objective becomes more challenging, and the performance gap between XOR-SMOO and other SOTA methods widens, further highlighting the advantage of our approach.}
    \label{fig:supply_pareto}
\end{figure*}

\section{Conclusion}
We proposed \method, a series of novel and efficient algorithms for solving Stochastic Multi-Objective Optimization (SMOO) problems with constant approximation guarantees.  
Our method reduces the original highly intractable (\#P-hard) problem to a sequence of satisfiability queries, augmented with randomized XOR constraints to handle probabilistic reasoning effectively. Through this reduction, \method systematically explores the discretized objective space and identifies achievable trade-offs among competing objectives.
\method obtains $\gamma$-approximate Pareto frontiers for SMOO problems with high probability $1-\delta$, via querying SAT oracles poly-log times in $\gamma$ and $\delta$. 
Experiments on real-world scenarios demonstrate the effectiveness of \method in discovering high-quality, diverse, and evenly-distributed Pareto frontiers. Compared to SOTA baselines, \method achieves significantly better performance, particularly as problem complexity increases. These results highlight the practical potential of combining effective probabilistic inference with stochastic multi-objective optimization.

\section*{Acknowledgement}

This research was supported by NSF grant CCF-1918327,
NSF Career Award IIS-2339844, DOE Fusion Energy
Science grant: DE-SC0024583.

\bibliography{multi-obj,ref}
\newpage
\begin{appendices}
\section{Experimental Details}
\label{app:exp-details}

\subsection{Baselines}
We compare \method\ against several widely used multi-objective evolutionary algorithms:
AGE-MOEA~\cite{age-moea}, 
NSGA-II~\cite{nsga2}, 
RVEA~\cite{rvea}, 
C-TAEA~\cite{c-taea}, and 
SMS-EMOA~\cite{sms-emoa}, 
all implemented using PyMOO~\cite{pymoo}\footnote{PyMOO: \url{https://pymoo.org/}}.  

Since objective evaluation involves model counting, baseline methods rely on external solvers:
GANAK\footnote{GANAK: \url{https://github.com/meelgroup/ganak}} for unweighted model counting, and 
Toulbar2\footnote{Toulbar2: \url{https://github.com/toulbar2/toulbar2}} for weighted model counting. All methods are given a time limit of one hour per run.

\noindent \textbf{Hyperparameters.} All baseline evolutionary algorithms use the same configuration:
\begin{itemize}
    \item Sampling: integer random sampling,
    \item Population size: 40,
    \item Crossover: SBX (prob=1.0, $\eta=3.0$),
    \item Mutation: PM (prob=1.0, $\eta=3.0$),
    \item Duplicate elimination enabled,
    \item Reference directions: uniform with 12 partitions (two objectives).
\end{itemize}
Unless otherwise stated, default PyMOO population settings are used.
All experiments are repeated five independent times with different random seeds.

\subsection{\method\ Implementation}

\method reduces SMOO to a sequence of Boolean satisfiability queries augmented with XOR constraints. 
Linear constraints encode objective thresholds into constraint satisfaction problems. 
We implement \method\ using CPLEX Studio 22.11 with Python \texttt{docplex}.  

Within the one-hour time limit, XOR constraints are added progressively.
Among all explored configurations, solutions associated with the largest number of XOR constraints (i.e., highest counting precision) are selected as the final output.

\subsection{Code Availability}

All code for reproducing the experiments is available at:
\url{https://anonymous.4open.science/r/XOR-SMOO/}.

\bigskip
\subsection{Scenario 1: Road Network Strengthening under Seasonal Disruptions}
\label{app:seasonal_details}

\subsubsection{Road Network Construction}

The road network is extracted from OpenStreetMap\footnote{OpenStreetMap: \url{https://wiki.openstreetmap.org/wiki/OSMnx}} 
centered at Central Park, NY, with radii 500m, 1000m, and 1500m. 
The corresponding graph sizes (\#nodes, \#edges) are:
$(20,30)$, $(181,345)$, and $(509,1044)$. For each instance, source $v_{\mathrm{src}} \in V$ and target nodes $v_{\mathrm{tgt}} \in V$ are selected randomly. Reachability is defined within hop limits $T=8,10,12$, respectively.

\subsubsection{Event Construction and Distribution}

Disruption events are derived from historical Meteostat\footnote{Meteostat: \url{https://meteostat.net/}} weather data in New York. 
We define four base event types: SnowDay, HeavyRainDay, HighWindDay, and HeatDay. 
For each type, we estimate two empirical probabilities:
$p_w$ (winter-like context) and $p_s$ (summer-like context).

To scale stochastic dimension, we replicate these base events to construct $K$ binary events, where $K$ increases with graph size.
Each event inherits $(p_w,p_s)$ with small perturbation.
Base types are divided into two families:
\[
\text{Winter}=\{\text{SnowDay, HighWindDay}\}, \quad
\text{Summer}=\{\text{HeatDay, HeavyRainDay}\}.
\]

Each event $s_i$ affects a localized subset of edges.
Edges are partitioned spatially, and each event samples affected edges using hub-based, contiguous, or random patterns.

To introduce correlation, we add $R$ latent binary regime variables
\[
\mathbf{Z}=(Z_1,\dots,Z_R).
\]
Each event selects one or two regime parents.
The joint distribution factorizes as
\[
\Pr(\mathbf{Z},\mathbf{s})
=
\prod_{r=1}^{R}\Pr(Z_r)
\prod_{i=1}^{K}\Pr(s_i \mid \mathbf{Z}_{\mathrm{parent}(i)}).
\]
For radii 500m, 1000m, and 1500m, we use
$(K,R)=(12,6)$, $(30,12)$, and $(50,20)$.

\subsubsection{Boolean Encoding of Connectivity}

For each edge $e \in E$, we introduce decision variables $x_e$ and event variables $s_i$.
Event $s_i$ disables edges in $E_i \subseteq E$ unless strengthened.

An edge is operational if
\[
u_e = x_e \;\lor\; \bigwedge_{i: e \in E_i} \neg s_i.
\]

Reachability within $T$ hops is encoded with auxiliary variables $r_{v,k}$:
\[
r_{v,k} \leftrightarrow 
\bigvee_{(u,v)\in E}
\left(r_{u,k-1} \land u_{(u,v)}\right),
\]
with base condition $r_{v_{\mathrm{src}},0}=1$.
Connectivity is defined as $r_{v_{\mathrm{tgt}},T}$.

\subsubsection{UAI Encoding of the Objective}

The SMOO problem is
\[
\max_{\mathbf{x}} 
\Big(
\sum_{\mathbf{s}} \Pr_{\text{summer}}(\mathbf{s})  
\mathbb{I}\big[\text{$v_{\mathrm{src}},v_{\mathrm{tgt}}$ connected} \mid \mathbf{s}, \mathbf{x}\big],
\sum_{\mathbf{s}} \Pr_{\text{winter}}(\mathbf{s})  
\mathbb{I}\big[\text{$v_{\mathrm{src}},v_{\mathrm{tgt}}$ connected} \mid \mathbf{s}, \mathbf{x}\big]
\Big).
\]

For fixed $\mathbf{x}$, each objective reduces to weighted model counting.

\subsubsection{\method: Reducing Weighted Counting to Unweighted Counting}

Probabilities in the UAI file are discretized into integers by multiplying by a scalar and rounding.
Let
\[
\text{val}(\mathbf{s},\mathbf{x})
=
\Pr(\mathbf{s})
\mathbb{I}\big[\text{$v_{\mathrm{src}},v_{\mathrm{tgt}}$ connected}\big]
\in \mathbb{Z}_{\ge 0}.
\]

Introduce binary counter variables $\mathbf{b}$ encoding integers in $[0,B)$ with $B \ge \text{val}(\mathbf{s},\mathbf{x})$.
Then
\[
\text{val}(\mathbf{s},\mathbf{x})
=
\sum_{\mathbf{b}}
\mathbb{I}\big[\text{binary\_value}(\mathbf{b}) < \text{val}(\mathbf{s},\mathbf{x})\big].
\]

Thus,
\[
\sum_{\mathbf{s}} \Pr(\mathbf{s})
\mathbb{I}\big[\text{$v_{\mathrm{src}},v_{\mathrm{tgt}}$ connected}\big]
=
\sum_{\mathbf{s},\mathbf{b}}
\mathbb{I}\big[
\text{binary\_value}(\mathbf{b}) < \text{val}(\mathbf{s},\mathbf{x})
\big],
\]
which is an unweighted model counting problem.

\bigskip
\subsection{Scenario 2: Flexible Supply Chain Network Design}
\label{app:supply_details}

\subsubsection{Network Construction}

Networks are derived from TSPLIB\footnote{TSPLIB: \url{http://comopt.ifi.uni-heidelberg.de/software/TSPLIB95/}} 
instances \texttt{Burma7}, \texttt{Burma14}, and \texttt{Ulysses16}.
Graph sizes are
\[
(7,42), \quad (14,182), \quad (16,240).
\]

Supplier $v_{\mathrm{src}}$ and demander $v_{\mathrm{tgt}}$ are selected randomly.

\subsubsection{CNF Encoding and Unweighted Counting}

For each route $e$, we introduce activation variables $x_e$ and shipment variables $f_e$.
Feasibility requires:
(i) $f_e \rightarrow x_e$,
(ii) flow conservation at intermediate nodes,
(iii) exactly one unit of net flow from $v_{\mathrm{src}}$ to $v_{\mathrm{tgt}}$.

All constraints are encoded in CNF.
For fixed $\mathbf{x}$, the number of satisfying assignments over $\mathbf{f}$ equals
\[
F(\mathbf{x})
=
\sum_{\mathbf{f}}
\mathbb{I}[\mathbf{f} \text{ valid under } \mathbf{x}].
\]

Thus, delivery flexibility reduces to unweighted model counting.
The final SMOO objective is
\[
\max_{\mathbf{x}}
\Big(
\widetilde{F}(\mathbf{x}),
-
\sum_{e \in E} d_e x_e
\Big).
\]

\end{appendices}
\end{document}